\documentclass{article}

\usepackage{microtype}
\usepackage{graphicx}
\usepackage{subfigure}
\usepackage{booktabs} 

\usepackage[accepted]{icml2019}

\icmltitlerunning{Fair Regression: Quantitative Definitions and Reduction-based Algorithms}

\makeatletter
\renewcommand{\paragraph}{%
  \@startsection{paragraph}{4}%
  {\z@}{.25ex \@plus 1ex \@minus .2ex}{-1em}%
  {\normalfont\normalsize\bfseries}%
}
\makeatother

\usepackage{enumitem}
\usepackage{algpseudocode}

\usepackage{amsmath,amssymb,amsthm}
\usepackage{eucal}
\usepackage{mathtools,xspace}
\usepackage{accents}
\usepackage{bbm}
\usepackage{xspace}
\usepackage[dvipsnames]{xcolor}
\usepackage{comment}
\newtheorem{example}{Example}

\newtheorem{lemma}{Lemma}
\newtheorem{theorem}{Theorem}
\newtheorem{assume}{Assumption}

\usepackage[]{color-edits}

\newcommand{\card}[1]{\lvert#1\rvert}

\newcommand{\set}[1]{\{#1\}}

\newcommand{\braces}[1]{\{#1\}}

\newcommand{\bigBracks}[1]{\bigl[#1\bigr]}
\newcommand{\BigBracks}[1]{\Bigl[#1\Bigr]}
\newcommand{\biggBracks}[1]{\biggl[#1\biggr]}
\newcommand{\BiggBracks}[1]{\Biggl[#1\Biggr]}
\newcommand{\Bracks}[1]{\left[#1\right]}
\newcommand{\parens}[1]{(#1)}
\newcommand{\Parens}[1]{\left(#1\right)}
\newcommand{\bigParens}[1]{\bigl(#1\bigr)}
\newcommand{\BigParens}[1]{\Bigl(#1\Bigr)}
\newcommand{\biggParens}[1]{\biggl(#1\biggr)}

\newcommand{\given}{\mathbin{\vert}}
\newcommand{\bigGiven}{\mathbin{\bigm\vert}}

\newcommand{\norm}[1]{\lVert#1\rVert}

\newcommand{\abs}[1]{\left\lvert#1\right\rvert}
\newcommand{\Abs}[1]{\left\lvert#1\right\rvert}
\newcommand{\bigAbs}[1]{\bigl\lvert#1\bigr\rvert}
\newcommand{\BigAbs}[1]{\Bigl\lvert#1\Bigr\rvert}

\newcommand{\floors}[1]{\lfloor#1\rfloor}
\newcommand{\ceils}[1]{\lceil#1\rceil}

\newcommand{\Ehat}{\widehat{\mathbb{E}}}
\newcommand{\Phat}{\widehat{\mathbb{P}}}
\renewcommand{\P}{\mathbb{P}}

\newcommand{\cost}{\textup{\mdseries cost}}
\newcommand{\hcost}{\widehat{\cost}}
\newcommand{\loss}{\textup{\mdseries loss}}
\newcommand{\hloss}{\widehat{\loss}}

\newcommand{\E}{\ensuremath{\mathbb{E}}}
\newcommand{\calH}{\ensuremath{\mathcal{H}}}
\newcommand{\calF}{\ensuremath{\mathcal{F}}}
\newcommand{\calX}{\ensuremath{\mathcal{X}}}
\newcommand{\calY}{\ensuremath{\mathcal{Y}}}
\newcommand{\calZ}{\ensuremath{\mathcal{Z}}}
\newcommand{\calA}{\ensuremath{\mathcal{A}}}
\newcommand{\calU}{\ensuremath{\mathcal{U}}}
\newcommand{\calJ}{\ensuremath{\mathcal{J}}}
\newcommand{\calG}{\ensuremath{\mathcal{G}}}
\newcommand{\calS}{\ensuremath{\mathcal{S}}}

\newcommand{\hy}{\ensuremath{\hat{y}}}
\newcommand{\heps}{\widehat{\eps}}
\newcommand{\hveps}{\widehat{\veps}}
\DeclareMathOperator*{\argmin}{arg\,min}
\DeclareMathOperator*{\argmax}{arg\,max}
\newcommand{\ind}{\mathbf{1}}
\newcommand{\eps}{\varepsilon}

\newcommand{\R}{\ensuremath{\mathbb{R}}}

\newcommand{\hzeta}{\widehat{\zeta}}
\newcommand{\hvzeta}{\widehat{\boldsymbol{\zeta}}}
\renewcommand{\phi}{\varphi}

\newcommand{\vlambda}{\boldsymbol{\lambda}}

\definecolor{purple}{rgb}{0.6, 0.4, 0.8}
\definecolor{red}{rgb}{1.0, 0.03, 0.0}
\definecolor{orange}{rgb}{0.93, 0.57, 0.13}
\definecolor{yellow}{rgb}{0.94, 0.88, 0.19}
\definecolor{green}{rgb}{0.0, 0.55, 0.55}
\definecolor{blue}{rgb}{0.0, 0.75, 1.0}

\renewcommand{\sp}{\textup{SP}\xspace}
\newcommand{\gammaE}{\gamma^{\textup{BGL}}}
\newcommand{\hgammaE}{\hgamma^{\textup{BGL}}}
\newcommand{\hvgammaE}{\hvgamma^{\textup{BGL}}}
\newcommand{\LE}{L^{\textup{BGL}}}
\newcommand{\underf}{\underaccent{\bar}{f}}
\newcommand{\underQ}{\underline{\smash{Q}\!\!}\,\,}
\newcommand{\underF}{\underaccent{\bar}{\calF}}
\newcommand{\ty}{\tilde{y}}
\newcommand{\tcalY}{\tilde{\calY}}
\newcommand{\tY}{\tilde{Y}}
\newcommand{\undery}{\underaccent{\bar}{y}}
\newcommand{\underY}{\underaccent{\bar}{Y}}
\newcommand{\cardZ}{N}
\newcommand{\cardA}{\card{\calA}}

\newcommand{\vzero}{\mathbf{0}}

\newcommand{\hvgamma}{\widehat{\boldsymbol{\gamma}}}
\newcommand{\vgamma}{\boldsymbol{\gamma}}
\newcommand{\hvlambda}{\widehat{\boldsymbol{\lambda}}}

\newcommand{\ve}{\mathbf{e}}

\newcommand{\hQ}{\widehat{Q}}
\newcommand{\vtheta}{\boldsymbol{\theta}}
\newcommand{\hgamma}{\widehat{\gamma}}
\newcommand{\veps}{\boldsymbol{\eps}}

\newcommand{\Eq}[1]{Eq.~\eqref{eq:#1}}
\newcommand{\Eqs}[2]{Eqs.~\eqref{eq:#1} and~\eqref{eq:#2}}

\newcommand{\Alg}[1]{Algorithm~\ref{alg:#1}}
\newcommand{\Sec}[1]{Section~\ref{sec:#1}}

\newcommand{\Thm}[1]{Theorem~\ref{thm:#1}}
\newcommand{\Lemma}[1]{Lemma~\ref{lemma:#1}}
\newcommand{\nullval}{\textit{\mdseries null}}

\newcommand{\ep}{bounded group loss\xspace}
\newcommand{\EP}{Bounded Group Loss\xspace}
\newcommand{\epshort}{BGL\xspace}

\newcommand{\BestLambda}{\textsc{\mdseries Best}_{\vlambda}}
\newcommand{\BestH}{\textsc{\mdseries Best}_h}
\newcommand{\BestF}{\textsc{\mdseries Best}_f}

\newtheorem{definition}{Definition}

\newcommand{\prob}{\textrm{fair regression}\xspace}

\newcommand{\order}{O}
\newcommand{\otil}{\widetilde{O}}
\newcommand{\sign}{\ensuremath{\text{sign}}}

\newcommand{\Step}[1]{Step~\ref{step:#1}}
\newcommand{\StepsRange}[1]{Steps~\ref{step:#1:start}--\ref{step:#1:end}} 

\begin{document}

\twocolumn[
\icmltitle{Fair Regression: Quantitative Definitions and Reduction-based Algorithms}



\icmlsetsymbol{equal}{*}

\begin{icmlauthorlist}
\icmlauthor{Alekh Agarwal}{msrred}
\icmlauthor{Miroslav Dud\'ik}{msrnyc}
\icmlauthor{Zhiwei Steven Wu}{umn}
\end{icmlauthorlist}

\icmlaffiliation{msrnyc}{Microsoft Research, New York, NY}
\icmlaffiliation{msrred}{Microsoft Research, Redmond, WA}
\icmlaffiliation{umn}{University of Minnesota, Minneapolis, MN}

\icmlcorrespondingauthor{A.~Agarwal}{alekha@microsoft.com}
\icmlcorrespondingauthor{M.~Dud\'ik}{mdudik@microsoft.com}
\icmlcorrespondingauthor{Z.~S.~Wu}{zsw@umn.edu}

\icmlkeywords{Machine Learning, ICML}

\vskip 0.3in
]



\printAffiliationsAndNotice{}  

\setlength{\abovedisplayskip}{3pt}
\setlength{\belowdisplayskip}{3pt}
\setlength{\textfloatsep}{10pt plus 2pt minus 2pt}
\setlength{\dbltextfloatsep}{10pt plus 2pt minus 2pt}

\begin{abstract}
In this paper, we study the prediction of a real-valued target, such as a risk score or recidivism rate, while guaranteeing a quantitative notion of fairness with respect to a protected attribute such as gender or race. We call this class of problems \emph{fair regression}. We propose general schemes for fair regression under two notions of fairness: (1) statistical parity, which asks that the prediction be statistically independent of the protected attribute, and (2) bounded group loss, which asks that the prediction error restricted to any protected group remain below some pre-determined level. While we only study these two notions of fairness, our schemes are applicable to arbitrary Lipschitz-continuous losses, and so they encompass least-squares regression, logistic regression, quantile regression, and many other tasks. Our schemes only require access to standard risk minimization algorithms (such as standard classification or least-squares regression) while providing theoretical guarantees on the optimality and fairness of the obtained solutions. In addition to analyzing theoretical properties of our schemes, we empirically demonstrate their ability to uncover fairness--accuracy frontiers on several standard datasets.\looseness=-1
\end{abstract}

\section{Introduction}
\label{sec:intro}


As machine learning touches increasingly critical aspects of our life, including education, healthcare, criminal justice and lending, there is a growing focus to ensure that the algorithms treat various subpopulations fairly (see, e.g., \citealp{barocas16big}; \citealp{podesta2014big}; \citealp{CorbettDaviesGo18}; and references therein). These questions have been particularly extensively researched in the context of classification, where several quantitative measures of fairness have been proposed~\citep{berk,chouldechova,hardt16,kleinberg2017inherent}, leading to a variety of algorithms that aim
to satisfy them~(see, e.g., \citealp{CorbettDaviesGo18}, for an overview of the literature).\looseness=-1


These classifier-based formulations
appear to fit the settings where the decision space is discrete and small, such as accept/reject decisions in hiring, school admissions, or lending. However, in practice, the decision makers work with tools that estimate a continuous quantity, such as success on the job, GPA in the first year of college, or risk of default on a loan.
Predictions of these quantities are treated as scores, which are used by human decision makers, perhaps in the context of a partly automated workflow, to reach final decisions (see, e.g., \citealp{waters2014grade,federal2007report,compas2010,pcra2012}).
While, in principle, a fair classification tool could be used to recommend the yes/no decision directly, such tools are often resisted by practitioners, because they limit their autonomy, whereas ranking or scoring tools do not have this drawback~\cite{veale2018fairness}.
In such situations, it is desirable to work with real-valued scores that satisfy some notion of fairness.
%
%
%
%
Yet, despite ample motivation and use cases, the prior work on designing fair continuous predictors
is quite limited in its scope
compared with the generality of methods for fair classification~\citep[e.g.,][]{
hardt16,
agarwal2018reductions}.

This paper seeks to diminish this gap by developing efficient algorithms
for a substantially broader set of regression tasks and model classes than done before, in many cases
providing the first method with theoretical performance guarantees.

We consider the problem of predicting a real-valued target,
where the prediction quality is measured by any Lipschitz-continuous loss function. Each example contains a \emph{protected attribute}, such as race or gender, with respect to which we seek to guarantee fairness.
We study two definitions of fairness from previous literature: \emph{statistical parity} (\sp), which asks that the prediction be statistically independent of the protected attribute, and
\emph{\ep} (\epshort), which asks that the prediction error restricted to any protected group stay below some pre-determined level.
We define \prob as the task of minimizing the expected loss of our real-valued predictions, subject to either of these fairness constraints.
By choosing the appropriate loss, we obtain a wide range of standard prediction tasks including least-squares, logistic, Poisson, and quantile regression (with labels and predictions restricted to a bounded set to obtain Lipschitz continuity). While we seek to solve the regression tasks under fairness constraints, our schemes only require access to standard risk minimization algorithms such as standard classification or least-squares regression.\looseness=-1

Several prior works also seek predictors that exhibit some form of independence from the protected attribute similar to statistical parity. 
\citet{CaldersEtAl13}, \citet{johnson2016impartial} and~\citet{komiyama2018nonconvex} consider a more limited form of independence, expressed via a small number of moment constraints, such as lack of correlation, and design specific algorithms for linear least squares. \citet{berk} study notions of individual and group fairness specialized to linear regression. \citet{perez2017fair} seek zero correlation in a reproducing kernel Hilbert space (RKHS), which can capture statistical independence,
but it only yields predictors in the same RKHS and the loss is limited to least squares. \citet{KamishimaEtAl12} and \citet{FukuchiEtAl13} seek to fit a probabilistic model that satisfies statistical independence, but they do not present efficient algorithms or statistical guarantees. In contrast, we consider full statistical independence, arbitrary model classes and Lipschitz losses, and our algorithms are efficient and come with statistical guarantees.\looseness=-1

Our second fairness definition, bounded group loss, fits into the general framework
of \citet{alabi2018unleashing}, whose goal is to minimize a general function of group-wise prediction losses, but their algorithm is less efficient (albeit still polynomial), and they do not provide statistical guarantees.

We design a separate algorithm for each of the two fairness definitions.
For \epshort, our insight is that the problem of loss minimization subject to a loss bound in each subpopulation can be algorithmically reduced to a weighted loss minimization problem for which standard approaches exist. For \sp, the main obstacle is that the number of constraints is uncountable. Here, the main insight that allows us to design and analyze the algorithm is that if we discretize the real-valued prediction space, then the task of \prob can be reduced to cost-sensitive classification under certain constraints. We build on the recent work of \citet{agarwal2018reductions}, and use the special structure of our discretization scheme to develop several algorithms reducing to standard classification or regression problems without fairness constraints. We provide theoretical results to bound the computational cost, generalization error and fairness violation of the returned predictor for both of our fairness measures with arbitrary Lipschitz-continuous loss functions and with arbitrary regression-function classes of bounded complexity, again building on the analysis of~\citet{agarwal2018reductions}. Prior works in the regression setting lack such guarantees.


Empirically, we evaluate our method on several standard datasets, on the tasks of least-squares and logistic regression under statistical parity, with linear and tree-ensemble learners, and compare it with the unconstrained baselines as well as the technique of~\citet{johnson2016impartial}. Our method uncovers fairness--accuracy frontiers
and provides the first systematic scheme for enforcing fairness in a significantly broader class of learning problems than prior work.\looseness=-1

\textbf{Usage guidelines.} We envision the use of our algorithms in uncovering fairness--accuracy frontiers in a variety of applications. 
Any substantial tradeoffs along the frontier need to be analyzed. They might point to data issues requiring non-algorithmic interventions, such as gathering of additional (less biased) data or introduction of new features~\citep{chen2018my}.
As with other algorithmic fairness tools, in order to successfully use our algorithms in practice, it is essential to consider the societal context of the application~\cite{selbst2018fairness}.
In some contexts, the best fairness intervention might be to avoid a technological intervention altogether.\looseness=-1


\section{Problem Formulation}
\label{sec:setting}

We consider a general prediction setting where the training examples consist of triples $(X,A,Y)$, where $X \in \calX$ is a feature vector,
$A \in \calA$ is a protected attribute and $Y \in \calY\subseteq[0,1]$ is the label. Throughout, we focus on the protected attribute taking a small
number of discrete values, i.e., $\calA$ is finite, but $\calX$ is allowed to be continuous and high-dimensional. We make no specific assumptions about whether
the protected attribute is included in the feature vector $X$ or not; also the set of labels $\calY$ can be discrete (but embedded in $[0,1]$) or continuous.
Given a set of predictors $\calF$ containing functions $f:\calX\to [0,1]$, our goal is to find $f \in \calF$ which is accurate in predicting $Y$ given $X$
while satisfying some fairness condition such as statistical parity or \ep (formally defined below).
Note that the functions $f$ do not explicitly depend on $A$ unless it is included in $X$.\looseness=-1

The main departure from prior works on classification is that $Y$ as well as $f(X)$ are allowed to be real-valued rather than just categorical.
The accuracy
of a prediction $f(X)$ on a label $Y$ is measured by the loss $\ell(Y,f(X))$. The loss function
$\ell:\calY\times[0,1]\to [0,1]$ is required to be 1-Lipschitz under the $\ell_1$ norm,\footnote{Our algorithms primarily use covers for $y$ and $u$ such that $\ell(y,u)$ can be approximated using corresponding elements from the cover. We skip this generalization to keep presentation simple.} that is:
\[
\left|\ell(y,u) - \ell(y'\!,u')\right| \leq \abs{y-y'}+\abs{u-u'}\;\;\text{for all $y,y'\!,u,u'\!$.}
\]

\begin{example}[Least-squares regression]
\label{ex:lsreg}
The prediction of GPA in the first year of college
can be cast as a regression problem where the label $y$ is the normalized GPA so that $\calY = [0,1]$, and the error is measured by the square loss $\ell(y,f(x)) = (y-f(x))^2/2$. Since $y,f(x) \in [0,1]$, the loss is bounded and
1-Lipschitz.
\end{example}

\begin{example}[Logistic regression]
\label{ex:logreg}
    Consider a system for screening job applicants based on the likelihood of an offer upon interview.
    We train this system using past data of interviewed candidates where $X$ describes their features and $Y\in\set{0,1}$
    the hiring decision. The scoring function $f$ can be chosen to maximize the likelihood for
    the logistic model $p_{f}(Y=1\given x)=1/(1+e^{-f(x)})$.
    Since we require that $f(x)\in[0,1]$, in order to approximate the full range of probabilities, we use a scaled and shifted version
    $p_{f}(Y=1\given x)=1/\bigParens{1+e^{-C(2f(x)-1)}}$ with some $C>1$, giving probabilities in the range $[1/(1+e^C),\;1/(1+e^{-C})]$.
    The loss is a rescaled version of the negative log likelihood to ensure the boundedness and 1-Lipschitz conditions: $\ell(y,f(x)) = \log\bigParens{1 + e^{-C(2y-1)(2f(x)-1)}}/\bigParens{2\log\parens{1+e^C}}$.
    Here
    the label is binary, but the prediction is real-valued.
\end{example}

\subsection{Fairness Definitions}
\label{sec:defs}

%

We consider two quantitative definitions of fairness appearing in prior work on fair classification and regression.

The first definition, called \emph{statistical} (or demographic) \emph{parity},
says that the prediction should be independent of the protected attribute. In classification, it corresponds to the practice of affirmative action
(see, e.g., \citealp{HolzerNe06}, and references therein) and it is also invoked to address disparate impact under the US Equal Employment Opportunity Commission's ``four-fifths rule,'' which requires that the ``selection rate for any race, sex, or ethnic group [must be at least] four-fifths (4/5) (or eighty percent) of the rate for the group with the highest rate.''\footnote{See the Uniform Guidelines on Employment Selection Procedures, 29 C.F.R. \S 1607.4(D) (2015).}
%
%
\begin{definition}[Statistical parity---\sp]
A predictor $f$ satisfies statistical parity under a distribution over $(X,A,Y)$ if $f(X)$ is independent of the protected attribute $A$. Since $f(X)\in[0,1]$, this is equivalent to $\P[f(X)\ge z\given A=a]=\P[f(X)\ge z]$ for all $a\in\calA$ and $z\in[0,1]$.%
\footnote{%
A standard definition of statistical independence requires that $\P[f(X)\in S\given A=a]=\P[f(X)\in S]$ for all measurable sets $S$.
Since $f(X)$ is a real-valued random variable under Borel $\sigma$-algebra, it is fully characterized by its cumulative distribution function, and
so it suffices to consider sets $S=[0,z]$ for $z\in[0,1]$ (see, e.g., Theorem~10.49 of \mbox{\citealp{aliprantis2006infinite}}).}
\label{defn:dp}
\end{definition}

The characterization through the properties of
the CDF of $f(X)$
is particularly useful when $f(X)$ can take any real values in $[0,1]$, because it allows us to design efficient algorithms.
It also makes it obvious that if $f$ satisfies \sp, then any classifier induced by thresholding $f$ will also satisfy \sp.

Our second fairness definition, called \emph{\ep}, formalizes the requirement that the predictor's loss remain below some acceptable level for each protected group. In settings such as speech or face recognition, this corresponds to the requirement that all groups receive good service (cf.~\citealp{buolamwini2018gender}). In other settings, such as lending and hiring, it aims to prevent situations when the predictor has a high error on some of the groups (cf.~Section~3.3 of~\citealp{CorbettDaviesGo18}).\looseness=-1
%
\begin{definition}[Bounded group loss---\epshort]
A predictor $f$ satisfies \ep at level $\zeta$ under a distribution over $(X,A,Y)$ if $\E[\ell(Y,f(X))\given A=a]\le\zeta$ for all $a\in\calA$.
\end{definition}
%
Hence, \prob with \epshort minimizes the overall loss, while controlling the worst loss on any protected group. By Lagrangian duality, this is equivalent to minimizing the worst loss on any group while maintaining good overall loss (referred to as max-min fairness).
Unlike
\emph{overall accuracy equality} in classification~\cite{dieterich2016compas}, which requires the losses
on all groups to be equal, \epshort does not force an artificial decrease in performance on every group just to match the hardest-to-predict group.
\epshort can be used as a diagnostic for the potential shortcomings of a chosen featurization or dataset. If it is not possible to achieve a loss below $\zeta$ on some group, then to achieve fairness we need to collect more data for that group, or develop more informative features for individuals in that group.\looseness=-1

\subsection{Fair Regression}
\label{sec:fair-reg}

We begin by defining the problem of fair regression as the minimization of the expected loss $\E[\ell(Y,f(X))]$ over $f \in \calF$, while guaranteeing \sp or \epshort. However, to achieve better fairness--accuracy tradeoffs we then generalize this to the case of randomized predictors.

\textbf{Statistical parity.} Similar to prior works on fair classification~\mbox{\cite{agarwal2018reductions}}, it is frequently desirable to have a tunable knob for navigating the fairness-accuracy tradeoff, such as $\zeta$ in the definition of \ep. To allow such a tradeoff in \sp, we consider slack parameters $\eps_a$ for each attribute and define the fair regression task under \sp as
\begin{align}
\notag
   &\min_{f \in \calF} \E\left[\ell(Y,f(X))\right]
   \quad\text{such that }\forall a\in\calA, z\in[0,1]{:}
\\
\label{eq:fair-reg:sp:0}
   &\BigAbs{\P[f(X)\ge z\given A=a] - \P[f(X)\ge z]}
   \leq \eps_a.
\end{align}
The slack $\eps_a$ bounds the allowed departure of the CDF of $f(X)$ conditional on $A=a$ from the CDF of $f(X)$. The difference between CDFs is measured in the $\ell_\infty$ norm corresponding
to the Kolmogorov-Smirnov statistic~\cite{lehmann2006testing}. Choosing different $\eps_a$ allows us to vary the strength of constraint across different protected groups.

\textbf{Bounded group loss.} In this case, the constrained optimization formulation follows directly from the definition. For the sake of flexibility, we allow specifying a different bound $\zeta_a$ for each attribute value, leading to the formulation
\begin{align}
\notag
   &\min_{f \in \calF} \E\bigBracks{\ell(Y,f(X))}
\\
\label{eq:fair-reg:ep:0}
   &\text{such that }\forall a\in\calA{:}\quad
   \E\bigBracks{\ell(Y,f(X))\bigGiven A=a}\le\zeta_a.
\end{align}

\textbf{Randomized predictors.}
Similar to fair classification, in order to achieve better fairness--accuracy tradeoffs, we consider randomized predictors which first pick $f$ according to some distribution $Q$ and then predict according to $f$. 
We first introduce additional notation for the objective and constraints appearing in~\eqref{eq:fair-reg:sp:0} and~\eqref{eq:fair-reg:ep:0}:
%
\begin{align*}
&\loss(f)\coloneqq
  \E[\ell(Y, f(X))],
\\
&\gamma_{a,z}(f)\coloneqq
  \P[f(X)\ge z\given A=a] - \P[f(X)\ge z].
\\
&\gammaE_{a}(f)\coloneqq
  \E\bigBracks{\ell(Y,f(X))\bigGiven A=a}.
\end{align*}
For a randomized predictor represented by a distribution~$Q$, we have $\loss(Q) = \sum_f Q(f)\loss(f)$, $\gamma_{a,z}(Q)= \sum_{f} Q(f)\gamma_{a,z}(f)$, and $\gammaE_{a}(Q)= \sum_{f} Q(f)\gammaE_{a}(f)$.

Thus, for \sp we seek to solve
\begin{align}
\label{eq:fair-reg:sp}
&
  \min_{\mathclap{Q\in\Delta(\calF)}}\;\loss(Q)
\text{ s.t. }
  \bigAbs{\gamma_{a,z}(Q)}\le\eps_a
\;\;
  \text{$\forall a\in\calA, z\in[0,1]$,}\!
\intertext{%
where $\Delta(\calF)$ is the set of all probability distributions over $\calF$. For \ep,
we similarly seek to solve}
\label{eq:fair-reg:ep}
&
 \min_{\mathclap{Q\in\Delta(\calF)}}\;\loss(Q)
\text{ s.t. }
  \gammaE_{a}(Q)\le\zeta_a
\;\;
  \text{$\forall a\in\calA$.}
\end{align}
%

\section{Supervised Learning Oracles}
\label{sec:oracles}

In this paper, we show how to transform the fair regression problem into three standard learning problems: cost-sensitive classification, weighted least-squares regression, or weighted risk minimization under $\ell$ (without fairness constraints). All of these learning problems allow different costs per example, which helps incorporate fairness. The specific algorithms to solve these tasks are termed \emph{supervised learning oracles}. These oracles are typically available for representations where regression or classification \emph{without fairness constraints} can be solved, and we show some typical examples in our empirical evaluation.

\textbf{(1) Risk minimization under $\ell$.} This is the most natural oracle  as it implements loss minimization without fairness constraints. Given a dataset
    $\set{(W_i,X_i,Y_i)}_{i=1}^n$ where $W_i$ are non-negative weights, the oracle returns $f\in\calF$ that minimizes the weighted empirical risk: $\sum_{i=1}^n W_i\ell(Y_i,f(X_i))$.

\textbf{(2) Square loss minimization.} Even when the accuracy is measured by $\ell$, we typically have access to a weighted least-squares learner for the same class $\calF$. This oracle takes the data $\set{(W_i,X_i,Y_i)}_{i=1}^n$ and returns $f\in\calF$ that minimizes the weighted squared loss: $\sum_{i=1}^n W_i \bigParens{Y_i-f(X_i)}^2$.

\textbf{(3) Cost-sensitive classification (CS).}
Our third type of oracle optimizes over classifiers $h:\calX'\to\set{0,1}$ from some class $\calH$. As input, we are given a dataset $\set{(X'_i,C_i)}_{i=1}^n$, where $X'_i$ is a feature vector and $C_i$ indicates the difference between the cost (i.e., the loss) of predicting 1 versus 0; positive $C_i$ means that 0 is favored, negative $C_i$ means that 1 is favored. The goal is to find a classifier $h\in\calH$, which minimizes the empirical cost relative to the cost of predicting all zeros: $\sum_{i=1}^n C_i h(X'_i)$.

CS reduces to \emph{weighted binary classification} on the data
$\set{(W_i,X'_i,Y_i)}_{i=1}^n$ with $Y_i=\ind\set{C_i\le 0}$ and $W_i=\abs{C_i}$, where we minimize
$\sum_{i=1}^n W_i\ind\braces{h(X'_i)\ne Y_i}$. Weighted classification oracles exist for many classifier families $\calH$.

In this paper we consider classifiers obtained by thresholding regressors $f \in \calF$. We define $\calX' = \calX\times \R$ where the new feature specifies a threshold. Our classifiers act on $x' = (x,z)$ and predict $h_f(x,z) = \ind\set{f(x) \geq z}$. This structure of classifiers naturally arises from the \sp constraints. We assume access to a CS oracle for $\calH=\set{h_f:\:f\in\calF}$. While cost-sensitive learners for this representation might not be available off the shelf, learners based on optimization, such as (stochastic) gradient-based learners, can usually be adapted to this structure. In particular, it is easy to adapt learners for logistic regression, SVMs or neural nets.

\section{Fair Regression under Statistical Parity}

We next show how to solve the fair regression problem~\eqref{eq:fair-reg:sp} using
a CS oracle.
We begin by recasting the problem~\eqref{eq:fair-reg:sp} as a \emph{constrained (and cost-sensitive) classification problem},
which we then solve via the reduction approach of \citet{agarwal2018reductions}, by repeatedly invoking the CS oracle.\looseness=-1

We proceed in two steps. First we discretize our prediction space and show that a loss function in the discretized space approximates our original loss well, owing to its Lipschitz continuity. We then show how the fair regression problem in this discretized space can be turned into a constrained classification problem, which we solve via reduction.

\subsection{Discretization}

We discretize both arguments of the loss function $\ell$. Let $N$ denote the size of the discretization grid for the second argument, let $\alpha=1/N$ denote its granularity, and let  $\calZ=\set{j\alpha:\:j=1,\dotsc,N}$ denote the grid itself. Let $\tcalY$ be the $\frac{\alpha}{2}$-cover of $\calY$, i.e., $\tcalY\subseteq\calY$ such that: (1) for any $y\in\calY$ there exists $\ty\in\tcalY$ such that $\abs{y-\ty}\le\frac{\alpha}{2}$, and (2) for any $\ty,\ty'\in\tcalY$, we have $\abs{\ty-\ty'}>\frac{\alpha}{2}$. Proceeding left-to-right within $\calY$, it is always possible to construct $\tcalY$ such that $\card{\tcalY}\le 2N$. We define the discretized loss as a piece-wise constant approximation of $\ell$:
\begin{equation}
\label{eq:discretized}
  \ell_{\alpha}(y,u)
\coloneqq
  \ell\Parens{\undery,\floors{u}_\alpha\!+\tfrac{\alpha}{2}}
\end{equation}
where $\undery$ is the smallest $\ty\in\tcalY$ such that $\abs{y-\ty}\le\frac{\alpha}{2}$,
and $\floors{u}_\alpha$ rounds down $u$ to the nearest integer multiple of $\alpha$. We use the convention $\ell(y,u)=\ell(y,1)$ for $u\ge 1$. Owing to the Lipschitz continuity of $\ell$, it follows that
\begin{align}
&\bigAbs{\ell(y,u) - \ell_\alpha(y,u)} \leq \alpha.
\label{eq:disc-loss}
\end{align}
Thus, for suitably small $\alpha$, or equivalently large $N$, $\ell_\alpha$ provides a close approximation to the original loss function. 

Let $\loss_\alpha(f)\coloneqq\E[\ell_\alpha(Y,f(X))]$ denote the expected discretized loss. When optimizing this loss, it suffices to consider rounded-down variants of predictors. Specifically,
for ${f\in\calF}$, let $\underf(x)=\floors{f(x)}_\alpha$ denote its rounded-down version.
Then, by
the definition of $\ell_\alpha$, $\loss_\alpha(f)=\loss_\alpha(\underf)$. The advantage of rounded-down predictors is that to guarantee that they satisfy \sp, it suffices to consider the fairness constraints $\gamma_{a,z}(\underf)\le\eps_a$ across $z$ taken from the discretization grid $\calZ$. This is because for any $z\in[0,1]$,
\begin{align}
\notag
\gamma_{a,z}(\underf)
  &=\P[\underf(X)\ge z\given A=a] - \P[\underf(X)\ge z]
\\
  &=\P[\underf(X)\ge \bar{z}\given A=a] - \P[\underf(X)\ge\bar{z}],
\label{eq:gamma:apx}
\end{align}
where $\bar{z}=\ceils{z}_\alpha$ is the value of $z$ rounded up to the nearest integer multiple of $\alpha$. This allows us to replace the uncountable set of constraints indexed by 
$z\in[0,1]$ with the finite set indexed by $z \in \calZ$. Thus, denoting $\underF=\set{\underf:\:f\in\calF}$, we have argued that the solution of~\eqref{eq:fair-reg:sp}, can be approximated by
\begin{align}
\label{eq:fair-reg:sp:apx}
\min_{\mathclap{\underQ\in\Delta(\underF)}}\;\loss_\alpha(\underQ)
\text{ s.t. }
  \bigAbs{\gamma_{a,z}(\underQ)}\le\eps_a
\;\;
  \text{$\forall a\in\calA, z\in\calZ$.}
\end{align}
%
%
\begin{theorem}
\label{thm:sp:apx}
Let $Q^\star$ be any feasible point of~\eqref{eq:fair-reg:sp} and $\underQ^\star$ be the solution of~\eqref{eq:fair-reg:sp:apx}. Then $\loss(\underQ^\star)\le\loss(Q^\star)+\alpha$
and $\abs{\gamma_{a,z}(\underQ^\star)}\le\eps_a$ for all $a\in\calA$, $z\in[0,1]$.
\end{theorem}

\subsection{Reduction to Constrained Classification}

We next show that \eqref{eq:fair-reg:sp:apx} can be rewritten as a constrained classification
problem for the family of classifiers $\calH=\set{h_f:\:f\in\calF}$ defined
in \Sec{oracles}.

To turn regression loss $\ell$ into a cost-sensitive loss, we introduce the function
\begin{equation}
c(y,z)\coloneqq N\BigParens{\ell\Parens{y,z+\tfrac{\alpha}{2}}-\ell\Parens{y,z-\tfrac{\alpha}{2}}},
\label{eqn:costs}
\end{equation}
which takes values in $[-1,1]$, because $\ell$ is 1-Lipschitz and $\alpha=1/N$.
%
%
We also extend the $\gamma_{a,z}$ notation to $h_f$:
\[
  \gamma_{a,z}(h_f) = \E[h_f(X,z) | A = a] - \E[h_f(X,z)].
\]
Now given a distribution $D$ over $(X,A,Y)$, we define a distribution $D'$ over $(X',A,C)$ that additionally samples $Z \in \calZ$ uniformly at random and sets $X' = (X,Z)$ and $C = c(\underY,Z)$. Defining $\cost(h_f) := \E_{D'}[Ch_f(X')]$, we have the following useful lemma.

\begin{lemma}
\label{lemma:red:hf}
Given any distribution $D$ over $(X,A,Y)$ and any $f \in \calF$, the cost and constraints  satisfy $\cost(h_f)=\loss_\alpha(\underf)+c_0$, where $c_0$ is independent of $f$, and $\gamma_{a,z}(h_f)=\gamma_{a,z}(\underf)$ for all $a\in\calA$, $z\in\calZ$.
\end{lemma}

%
By linearity of expectation, the lemma implies analogous equalities also for distributions over $f$. Thus, in problem~\eqref{eq:fair-reg:sp:apx}, we can replace the optimization over $\underQ\in\Delta(\underF)$ with $Q \in\Delta(\calH)$. 
%
Notice that while we started from discretized regressors in problem~\eqref{eq:fair-reg:sp:apx}, Lemma~\ref{lemma:red:hf} allows us to work with the full classifier family $\set{h_f:\:f \in \calF}$, which is important as we typically only have computational oracles for non-discretized classes $\calH$ and $\calF$.
We next show how to solve an empirical version of this classification problem.\looseness=-1

\subsection{Algorithm and Generalization Bounds}
\label{sec:algo}

Let $\Ehat$ denote the empirical distribution over the data $(X_i,A_i,Y_i)$ and let $\E_Z$ denote a uniform distribution among the values in $\calZ$. Then
define the empirical versions of the cost and constraints:
\begin{align}
\label{eq:hcost}
   \hcost(h_f)&=\Ehat\BigBracks{\E_Z\bigBracks{c(\underY,Z)h_f(X,Z)}}
\\
\notag
   \hgamma_{a,z}(h_f)&=\Ehat\bigBracks{h_f(X,z) \bigGiven A=a} - \Ehat\bigBracks{h_f(X,z)}.
\end{align}
We are interested in the following empirical optimization problem, which is, according to  Lemma~\ref{lemma:red:hf}, an empirical approximation of the original problem~\eqref{eq:fair-reg:sp:apx}:
\begin{align}
\label{eq:prob:sp:empirical}
  \min_{\mathclap{Q\in\Delta(\calH)}}\;\hcost(Q)
\text{ s.t. }
  \bigAbs{\hgamma_{a,z}(Q)}\le\heps_a
\;\;
  \text{$\forall a\in\calA, z\in\calZ$.}
\end{align}
The slacks $\heps_a$ are slightly larger than $\eps_a$ to compensate for finite-sample errors in measuring constraint violations (more on that below). This problem is a special case of that studied by \citet{agarwal2018reductions} with a key difference.
Since the distribution of $Z$ is known, we can take expectation according to $Z$ rather than a sample, which leads
to substantially better estimates of constraint violations. Thus,
our objective uses a product of an empirical distribution over $(X,A,Y)$ with the uniform distribution over $Z$ rather than an i.i.d.\ sample as assumed by \citeauthor{agarwal2018reductions}.
However,
their algorithm and generalization bounds
still apply (as we show in our proofs).\looseness=-1

The algorithm begins by forming the Lagrangian with the primal variable $Q\in\Delta(\calH)$ and the dual variable $\vlambda$ with components
$\lambda_{a,z}^+,\lambda_{a,z}^-\in\R_+$, corresponding to the constraints $\hgamma_{a,z}(Q)\le \heps_{a}$ and $\hgamma_{a,z}(Q)\ge-\heps_{a}$:
\begin{align*}
  L(Q,\vlambda)
&=\hcost(Q)+\sum_{a,z}\Bigl[
            \lambda^+_{a,z}\bigParens{\hgamma_{a,z}(Q)-\heps_a}
\\[-12pt]
&\hspace{1.25in}
           +\lambda^-_{a,z}\bigParens{-\hgamma_{a,z}(Q)-\heps_a}\Bigr].
\end{align*}
It solves the saddle-point problem $\min_Q\max_{\vlambda} L(Q,\vlambda)$
over $Q\in\Delta(\calH)$ and $\vlambda\ge\vzero$, $\norm{\vlambda}_1\le B$,
%
%
by treating it as a two-player zero-sum game (see Algorithm~\ref{alg:regred:sp} for details).\looseness=-1

%

We bound the suboptimality and fairness of the returned solution
largely following their analysis. Let $R_n(\calH)$ denote the Rademacher complexity of $\calH$ (see Eq.~\ref{eqn:rademacher} in Appendix~\ref{app:thm:main:sp}). To state the bounds, recall an assumption from their paper on the setting of the empirical slacks $\heps_{a}$:
\begin{assume}
    There exist $C, C' > 0$ and $\beta \leq 1/2$ such that $R_n(\calH) \leq Cn^{-\beta}$ and $\heps_{a} = \eps_{a} + C'n_{a}^{-\beta}$, where $n_{a}$ is the number of samples with $A = a$.
\label{ass:samples}
\end{assume}

Under this assumption, we obtain the following guarantees.\footnote{%
The notation $\otil(\cdot)$ suppresses polynomial dependence on $\ln n$, $\ln\,\card{\calA}$, and $\ln(1/\delta)$.}
\begin{theorem}
\label{thm:main:sp}
Let Assumption~\ref{ass:samples} hold for $C' \geq 2C + 2 + \sqrt{2\ln(4\cardA\cardZ/\delta)}$, where $\delta > 0$. Let $Q^\star$ be any feasible distribution for the fair regression problem~\eqref{eq:fair-reg:sp}. Then Algorithm~\ref{alg:regred:sp} with $\nu \propto n^{-\beta}$, $B \propto n^\beta$, and $\cardZ\propto n^\beta$ terminates in $\order\left(n^{4\beta}\ln (n^\beta\card{\calA})\right)$ iterations and returns $\hQ$, which, when viewed as a distribution over $\underF$, satisfies with probability at least 1-$\delta$,
\begin{align*}
  \loss(\hQ)
&\leq \loss(Q^\star) + \otil(n^{-\beta})
\\
  \left|\gamma_{a,z}(\hQ)\right|
&\leq \eps_{a} + \otil(n_{a}^{-\beta})
\quad
\text{for all $a\in\calA$, $z\in[0,1]$.}
\end{align*}
\end{theorem}

Note that the bounds grow with the Rademacher complexity of $\calH$, rather than the complexity of the regressor class $\calF$. Since $\bigAbs{\E_Z[h_f(X,Z)] - f(X)} \leq \alpha$, it can be shown that $R_n(\calH) \geq R_n(\calF) - \alpha$, meaning that the classifiers induce a more complex class. The bound on $\smash{\loss(\hQ)}$ in Theorem~\ref{thm:main:sp} can be stated in terms of the tighter $R_n(\calF)$, but the constraints still deviate by $R_n(\calH)$, which we believe is unavoidable. However, if $\calF$ has a bounded pseudo-dimension, which always equals the VC dimension of $\calH$~\citep{anthony2009neural}, then the pseudo-dimension can be used to bound $R_n(\calH)$ (see Theorem~6 of~\citealp{BartlettMe02}).\looseness=-1

\begin{algorithm}[t!]
\caption{Fair regression with statistical parity}
\label{alg:regred:sp}
  \begin{algorithmic}[1]
    \Statex{Input:~~training examples $\braces{(X_i,Y_i,A_i)}_{i=1}^n$,}
    \Statex{~\hphantom{Input:}~slacks $\heps_{a}\in[0,1]$,
    bound $B$, threshold $\nu$}
\smallskip
    \Statex{Define best-response functions:}
    \Statex{~\hphantom{Input:}~$\BestH(\vlambda)\coloneqq\argmin_{h_f\in\calH} L(h_f,\vlambda)$}
    \Statex{~\hphantom{Input:}~$\BestLambda(Q)\coloneqq\argmax_{\vlambda\ge 0,\,\norm{\vlambda}_1\le B} L(Q,\vlambda)$}
\smallskip
    \State{Set learning rate $\eta=\nu/(8B)$}
    \State{Set $\vtheta_1^+=\vtheta_1^-=\vzero\in\R^{\card{\calA}\cardZ}$}
\smallskip
    \For{$t=1, 2, \ldots$}
\smallskip
\Statex{\hspace{\algorithmicindent}%
        \textit{// Compute $\vlambda_t$ from $\vtheta_t$ and find the best response $h_t$}}
    \ForAll{$a,z$ and $\sigma \in \{+,-\}$}
       \label{step:lambda:update:start}
\vspace{2pt}
    \State{$\lambda^\sigma_{t,a,z}\gets B\,
                \frac{\exp\braces{\theta^\sigma_{t,a,z}}}{1+\sum_{a,z}\bigBracks{\exp\braces{\theta^+_{t,a,z}}+\exp\braces{\theta^-_{t,a,z}}}}$}
\vspace{-2pt}
    \EndFor
       \label{step:lambda:update:end}
    \State{$h_t\gets\BestH(\vlambda_t)$}
       \label{step:best:response}
\smallskip
\Statex{\hspace{\algorithmicindent}%
        \textit{// Calculate the current approximate saddle point}}
    \State{$\hQ_t\gets\frac{1}{t}\sum_{t'=1}^t h_{t'},
      \quad
            \hvlambda_t\gets\frac{1}{t}\sum_{t'=1}^t \vlambda_{t'}$}
\smallskip
\Statex{\hspace{\algorithmicindent}%
        \textit{// Check the suboptimality of $(\hQ_t,\hvlambda_t)$}}
    \State{$\overline{\nu}\gets L\bigParens{\hQ_t,\BestLambda(\hQ_t)}\;-\;L(\hQ_t,\hvlambda_t)$}
       \label{step:conv:check:start}
       \label{step:L:bestL}
    \State{$\underline{\nu}\gets L(\hQ_t,\hvlambda_t)\;-\;L\bigParens{\BestH(\hvlambda_t),\hvlambda_t}$}
       \label{step:L:bestH}
    \State{\textbf{if} $\max\set{\overline{\nu},\underline{\nu}}\le\nu$ \textbf{then return} $\hQ_t$}
       \label{step:conv:check:end}
\smallskip
\Statex{\hspace{\algorithmicindent}%
        \textit{// Apply the exponentiated-gradient update}}
    \State{Set $\vtheta^\sigma_{t+1}=\vtheta^\sigma_t+\sigma\eta\hvgamma(h_t)-\eta\hveps$, for $\sigma \in \{+,-\}$}
       \label{step:theta:update}
    \EndFor
\end{algorithmic}
\end{algorithm}

\subsection{Efficient Implementation of \Alg{regred:sp}}
\label{sec:efficient}

It is not too difficult to show that each iteration of \Alg{regred:sp} can be implemented in time $\order(n\log n+\card{\calA}\cardZ)$ plus the complexity of two calls to $\BestH$, on which we focus here, while deferring the remaining details to Appendix~\ref{sec:app:impl}.

\textbf{Reduction to cost-sensitive classification.}
We first show how to implement $\BestH$ using the CS oracle.
Letting $\lambda_{a,j}=\lambda^+_{a,j}-\lambda^-_{a,j}$, and dropping terms independent of $h$, the minimization over $h$ only needs to be over $\hcost(h) + \sum_{a,z}\lambda_{a,z}\hgamma_{a,z}(h)$.
The first term is already defined as CS error with respect to $\Ehat$ and $\E_Z$ (see Eq.~\ref{eq:hcost}). Let $p_{a}\coloneqq\Phat[A=a]$. Then we show in Appendix~\ref{sec:app:impl} that the minimization over $L(h,\vlambda)$ is equivalent to minimizing
\begin{align}
\label{eq:c:lambda}
&
  \hcost(h)
  +
  \!\sum_{a,z}\lambda_{a,z}\hgamma_{a,z}(h)
\\
\notag
&=
  \Ehat\biggBracks{
  \E_Z\biggBracks{
  \biggParens{
  \underbrace{
     c(\underY,Z)+
     \frac{N\lambda_{A,Z}}{p_A} -\!\sum_a N\lambda_{a,Z}
  }_{c_{\vlambda}(\underY,A,Z)}
  }
  h(X,Z)
  }}.
\end{align}
This corresponds to a CS problem
with $nN$ instances $\set{(X'_{i,z},C_{i,z})}_{i\le n,\,z\in\calZ}$ defined as
\[
  X'_{i,z}=(X_i,z),
\quad
  C_{i,z}=c_{\vlambda}(\underY_i,A_i,z)
.
\]
The sum $\sum_a N\lambda_{a,Z}$ in the definition of $c_{\vlambda}$ can
be precalculated once for each value of $Z$ in the overall time $\order(N\card{\calA})$.
After that the construction of the dataset takes time $\order(nN)$.

Based on Assumption~\ref{ass:samples} and \Thm{main:sp}, we expect $N\propto n^{\beta}$, so
this reduction to cost-sensitive classification takes time $\Omega(n^{1+\beta})$ and creates a dataset of size $\Omega(n^{1+\beta})$. This is substantially larger than the original problem of size $n$, given the typical value $\beta\approx 1/2$. We next describe two alternatives that run faster and only create datasets of size $n$. 

\textbf{Reduction to least-squares regression.}
The main overhead in the CS reduction above comes from the summation over $z\in\calZ$, implicit in the expectation over $Z$ in Eq.~\eqref{eq:c:lambda}. In order to eliminate this overhead, suppose we have access to a function $g_{\vlambda}$
such that
\[
g_{\vlambda}(\tY,A,f(X))
=
\frac{1}{N} \sum_{z \in \calZ} c_{\vlambda}(\tY, A, z) h_f(X,z)
\]
for any $\tY\in\tcalY$, $A \in \calA$, $X \in \calX$ and $f\in\calF$. In Appendix~\ref{sec:app:impl:ls}, we show how to precalculate $g_{\vlambda}(\cdot,\cdot,\cdot)$
in time $\order(\card{\calA}\card{\tcalY} N)$. Then minimizing Eq.~\eqref{eq:c:lambda} over $h$ is equivalent to finding $\min_{f \in \calF} \sum_{i=1}^n g_{\vlambda}(\underY_i,A_i,f(X_i))$. We heuristically solve this problem by calling a least-squares oracle on a dataset of size $n$. To this end, we replace $g_{\vlambda}$ by the square loss with respect to specific targets $U_i$ (different from $Y_i$). As targets, we choose the minima of $g_{\vlambda}$ for each fixed $\underY_i$ and $A_i$, that is, $U_i\in\argmin_{u\in[0,1]} g_{\vlambda}(\underY_i,A_i,u)$. We seek to solve
\[
  \min_{f \in \calF}\,\sum_{i=1}^n (U_i-f(X_i))^2
.
\]
To obtain the values $U_i$ we first calculate $c_{\vlambda}(\ty,a,z)$ across all $\ty$, $a$, and $z$
in the overall time $\order(\card{\calA}\card{\tcalY} N)$. Then, using the definition of $g_{\vlambda}$, the minimizer
of $g_{\vlambda}(\ty,a,u)$ over $u$ can be found in time $\order(N)$ for each pair of $\ty$ and $a$, so
all the minimizers can be precalculated in time $\order(\card{\calA}\card{\tcalY} N)$. Thus, preparing the dataset for
the least-squares reduction takes time $\order(\card{\calA}\card{\tcalY} N)$ and the resulting regression dataset
is of size $n$. Since $\card{\tcalY}=O(N)$, and $N\propto n^{\beta}$. The running time of the reduction is $O(n\log n + \card{\calA}n^{2\beta})$, which is substantially faster than $\Omega(n^{1+\beta})$ for typical $\beta\approx 1/2$.

\textbf{Reduction to risk minimization under $\ell$.}
A similar heuristic as in the least-squares reduction can be used to reduce $\BestH$ to risk minimization under any loss $\ell(y,u)$ that is convex in $u$. The complexity of this reduction is identical to that of the least-squares reduction, but the resulting risk minimization problem might be better aligned with the original problem, yielding a potentially superior oracle as we will see in the experiments. (See Appendix~\ref{app:reduction:ell} for details.)%


\section{Fair Regression with \EP}
\label{sec:max-loss-bound}

We now turn attention to our second notion of fairness. We show
how to reduce fair regression with \ep to loss minimization under $\ell$ without the fairness constraints.

The approach still follows the scheme of \citet{agarwal2018reductions}, but
thanks to the matched loss function between the objective and
the constraints, fair regression can be reduced directly to regression,
without the need for discretization.
We first replace the problem~\eqref{eq:fair-reg:ep} by its empirical
version
\begin{equation}
\label{eq:fair-reg:ep:emp}
  \min_{\mathclap{Q\in\Delta(\calF)}}\;\hloss(Q)
\;\;
\text{s.t.}
\;\;
  \hgammaE_{a}(Q)\le\hzeta_a
\;\;
  \forall a\in\calA.
\end{equation}
We then form the Lagrangian with the primal variable $Q$ and the dual variable $\vlambda$ with components $\lambda_a\in\R_+$ corresponding to the constraints $\hgammaE_{a}(Q)\le\hzeta_a$:
\begin{align*}
  \LE(Q,\vlambda)
&=\hloss(Q)+\sum_a\lambda_a\BigParens{\hgammaE_a(Q)-\hzeta_a}.
\end{align*}
We give a detailed pseudocode for our approach in Algorithm~\ref{alg:regred:ep} in Appendix~\ref{sec:app:alg:ep}, and describe the main differences from Algorithm~\ref{alg:regred:sp} here. As before, the algorithm alternates between exponentiated gradient updates on $\vlambda$ and best responses for $Q$ to compute an approximate saddle point:
\begin{equation}
\label{eq:minmax:ep}
  \adjustlimits\min_{\;\;Q\in\Delta\;\;}\max_{\vlambda\ge\vzero,\,\norm{\vlambda}_1\le B} \LE(Q,\vlambda)
.
\end{equation}
The saddle point always exists. However, unlike the fair regression problem under \sp, the fair regression problem under \epshort, i.e., Eq.~\eqref{eq:fair-reg:ep:emp}, might be infeasible. Therefore, Algorithm~\ref{alg:regred:ep} explicitly checks whether the distribution $\hQ$ that it finds satisfies constraints of Eq.~\eqref{eq:fair-reg:ep:emp}.

The other main difference between Algorithms~\ref{alg:regred:sp} and~\ref{alg:regred:ep} is in the computation of the best response $f$ to any given $\vlambda$, which requires solving the problem
\[
\min_{f \in \calF}\BigBracks{\hloss(f) + \sum_a \lambda_a \hgammaE_a(f)}.
\]
Denoting by $n_a$ the number of samples with $A_i = a$,
this minimization can be written as
\[
\min_{f \in \calF}\biggBracks{\frac{1}{n}\sum_{i=1}^n \ell(Y_i, f(X_i)) +
\sum_a\frac{\lambda_a}{n_a}\,\sum_{i:\:A_i=a}\ell(Y_i,f(X_i))},
\]
which can be solved using one call to the weighted risk-minimization oracle.

We finish this section with the optimality and fairness guarantees for $\smash{\hQ}$ returned by Algorithm~\ref{alg:regred:ep}. We assume that $\smash{\hzeta_{a}}$ are set according to the Rademacher
complexity of $\calF$:
\begin{assume}
\label{ass:samples:ep}
    There exist $C, C' > 0$ and $\omega \leq 1/2$ such that $R_n(\calF) \leq Cn^{-\omega}$ and $\smash{\hzeta_{a}} = \zeta_{a} + C'n_{a}^{-\omega}$, where $n_{a}$ is the number of samples where $A = a$.
\end{assume}


\begin{theorem}
\label{thm:main:ep}
Let Assumption~\ref{ass:samples:ep} hold for $C' \geq 4C + 2 + \sqrt{2\ln(4\card{\calA}/\delta)}$, where $\delta > 0$. Then Algorithm~\ref{alg:regred:ep} with $\nu \propto n^{-\omega}$ and $B \propto n^\omega$ terminates in
$\order(n^{4\omega}\ln\,\card{\calA})$ iterations and returns $\smash{\hQ}$ such that, with probability at least 1-$\delta$, one of the following holds:
\begin{enumerate}[nosep,itemsep=6pt]
\item $\hQ\ne\nullval$ and, for any $Q^\star$ feasible in problem~\eqref{eq:fair-reg:ep},\\[6pt]
${}\quad\loss(\hQ)\leq \loss(Q^\star) + \otil(n^{-\omega})$\\
${}\quad\gammaE_a(\hQ)\leq \zeta_{a} + \otil(n_{a}^{-\omega})
 \quad
 \text{for all $a\in\calA$.}$
\item $\hQ=\nullval$ and problem \eqref{eq:fair-reg:ep} is infeasible.
\end{enumerate}
\end{theorem}


\section{Experiments}
\label{sec:experiments}

\begin{figure*}[!t]
  \includegraphics[width=\textwidth, height=8cm]{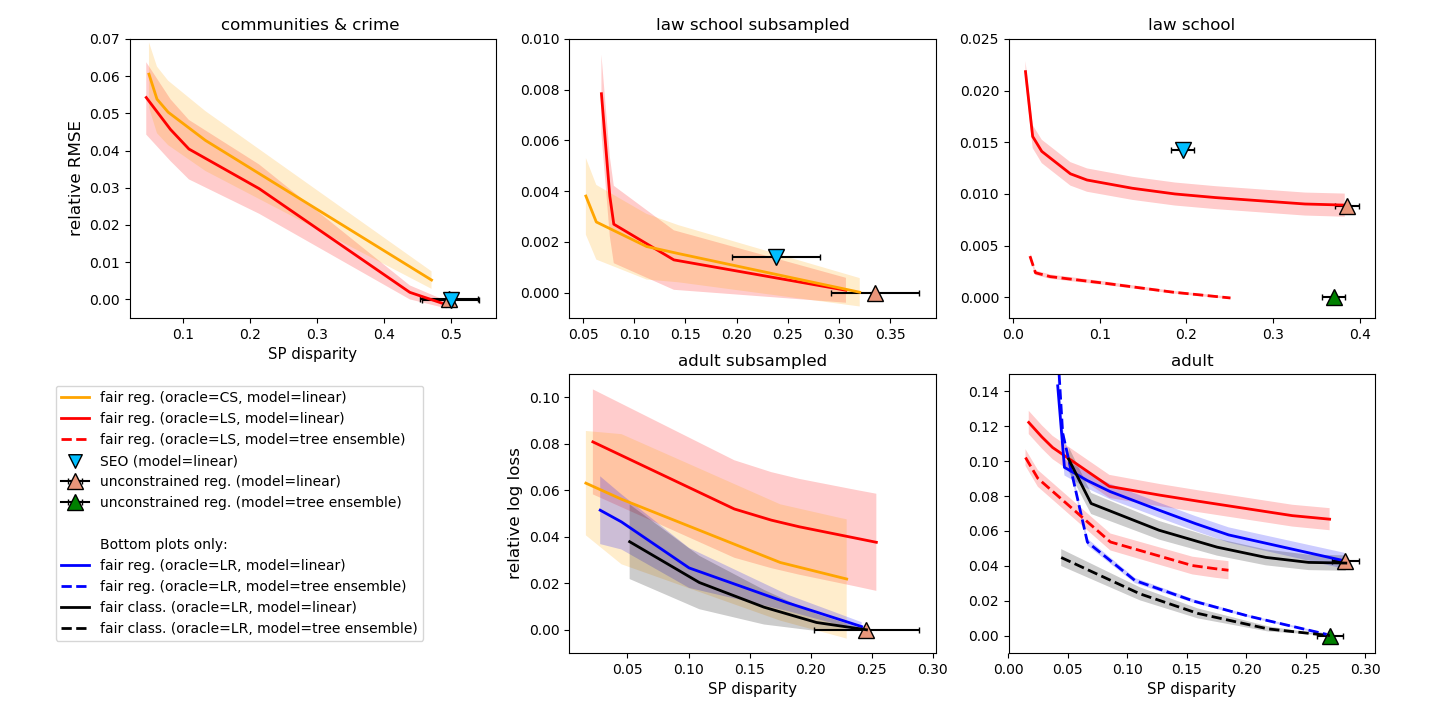}
\vspace{-8mm}
\caption{Relative test loss versus the worst constraint violation with respect
  to SP. Relative losses are computed by subtracting the smallest
  baseline loss from the actual loss. For our algorithm and fair classification we plot the
  convex envelope of the predictors obtained on training data at
  various accuracy--fairness tradeoffs. We show 95\% confidence bands
  for the relative loss of our method and fair classification, and also show 95\% confidence
  intervals for constraint violation
  (the same for all methods).
  Our method dominates or matches the baselines
  up to statistical uncertainty on all datasets except \emph{adult},
  where fair classification is slightly better.}
\label{fig:testset}
\end{figure*}

We evaluate our method on the tasks of least-squares regression and logistic regression
under statistical parity. We use the following three datasets:

\emph{Adult}: The adult income dataset~\cite{Lichman:2013}
  has 48,842 examples. The task is to predict the probability that
  an individual makes more than \$50k per year via logistic
  loss minimization, with gender as the protected attribute.

\emph{Law school}: Law School Admissions Council’s National
  Longitudinal Bar Passage Study~\cite{Wightman} has 20,649
  examples. The task is to predict a student’s GPA (normalized
  to $[0,1]$) via square loss minimization, with race as
  the protected attribute (white versus non-white).

\emph{Communities \& crime}: The dataset contains
  socio-economic, law enforcement, and crime data about
  communities in the US~\cite{communities} with 1,994 examples. The
  task is to predict the number of violent crimes per 100,000
  population (normalized to $[0, 1]$) via square loss
  minimization, with race as the protected attribute (whether the
  majority population of the community is white).

For the two larger datasets (\emph{adult} and \emph{law school}), we also
created smaller (subsampled) versions by picking random 2,000 points. Thus we ended up
with a total of five
datasets, and split each into 50\% for training and 50\% for testing.\looseness=-1

We ran Algorithm~\ref{alg:regred:sp} on each
training set over a range of constraint slack values
$\hat \eps$, with a fixed discretization grid of size 40:
$\calZ=\{1/40, 2/40, \ldots, 1\}$. Among the solutions for different $\hat\eps$,
we selected the ones on the Pareto front based on their training losses
and SP disparity $\max_{a,z}\{\hat \gamma_{a, z}\}$. We
then evaluated the selected predictors on the test set, and show the resulting Pareto front
in Figure~\ref{fig:testset}.
%

We ran our algorithm with the three types of
reductions from Section~\ref{sec:efficient}: reductions to cost-sensitive
(CS) oracles, least-squares
(LS) oracles, and logistic-loss minimization (LR) oracles.
Our CS oracle sought the linear model minimizing
weighted hinge-loss (as a surrogate for weighted classification error).
%
%
Because of unfavorable scaling
of the cost-sensitive problem sizes (see Section~\ref{sec:efficient}),
we only ran the CS oracle on the three small datasets.
We considered two variants of LS and LR oracles: linear learners from
scikit-learn \citep{scikit-learn}, and tree ensembles from XGBoost~\cite{xgb}.
Tree ensembles heavily overfitted smaller datasets, so we only show their performance
on two larger datasets. We only used LR oracles when the target loss was logistic,
whereas we used LS oracles across all datasets.
%

In addition to our algorithm, we also evaluated regression without any fairness constraints,
and two baselines from the fair classification and fair regression literature.

On the three datasets where the task was least-squares regression, we evaluated the \emph{full
  substantive equality of opportunity} (SEO) estimate of
\citet{johnson2016impartial}. It can be obtained in a closed form by solving for the
linear model that minimizes least-squares error while having zero correlation with
the protected attribute. In contrast, our method seeks
to limit not just correlation, but statistical dependence.

On the two datasets where the task was logistic regression, we ran the \emph{fair classification} (FC)
reduction of~\citet{agarwal2018reductions} with the same LR oracles as in our algorithm.
For this choice of oracles, the classifiers returned by FC
are implemented by logistic models and return real-valued scores, which we evaluated.
We ran FC across a range of trade-offs between classification accuracy
and statistical parity (in the classification sense) and show the resulting Pareto front.
Note that
FC only enforces statistical parity (SP) when the scores are thresholded at zero, whereas our method
enforces SP across all thresholds.

In Figure~\ref{fig:testset}, we see that all of our reductions are able to significantly reduce disparity, without strongly impacting the overall loss. On \emph{communities \& crime}, there is a more substantial accuracy--fairness tradeoff, which can  be used as a starting point to diagnose the data quality for the two racial subgroups.
Our methods dominate SEO in least-squares tasks, but are slightly worse than FC in logistic regression.
The difference is statistically significant only on \emph{adult}, where it points to the limitations of our LS and LR reduction heuristics.
However, for the most part, LR and LS reductions achieve tradeoffs on par with the CS reduction, and are substantially faster to run (see Appendix~\ref{app:exp}). The results on \emph{adult} and \emph{adult subsampled} suggest that reducing to a matching loss is preferable over reducing to another loss.\looseness=-1

In summary,
we have shown that our scheme efficiently handles a range of losses and regressor classes and, where possible, diminishes disparity while maintaining the overall accuracy. The emergence of FC
as a strong baseline for logistic regression suggests that our regression-based reduction heuristics can be further improved, which we leave open for future research.





\section*{Acknowledgements}
ZSW is supported in part by a Google Faculty Research Award, a J.P.
Morgan Faculty Award, and a Facebook Research Award. Part of this work
was completed while ZSW was at Microsoft Research-New York City.

\bibliography{arxiv.bbl}
\bibliographystyle{icml2019}

\appendix
\newpage
\onecolumn

\section{Proof of Lemma~\ref{lemma:red:hf}}

We begin by rewriting the loss $\ell_\alpha$ as a cost-sensitive classification loss. First, we use the telescoping trick to obtain
\begin{align*}
\ell_{\alpha}(y,u)
&
 =\ell\BigParens{\undery,\floors{u}_\alpha\!+\tfrac{\alpha}{2}}
\\
&
=\ell(\undery,\tfrac{\alpha}{2})
    +\sum_{z\in\calZ}
     \BigBracks{
     \underbrace{\ell\BigParens{\undery,z+\tfrac{\alpha}{2}}-\ell\BigParens{\undery,z-\tfrac{\alpha}{2}}
     }_{c(\undery,z)/N}
     }
     \ind\braces{u\ge z}.
\end{align*}

Now plugging in $u=\underf(x)$ and using the fact that for $z\in\calZ$, we have $\ind\set{\underf(x)\ge z}=\ind\set{f(x)\ge z}=h_f(x,z)$, we obtain
\begin{align}
\label{eq:loss:hf}
  \ell_\alpha(y,\underf(x))
&
  =\ell(\undery,\tfrac{\alpha}{2})+\frac{1}{N}\sum_{z\in\calZ} c(\undery,z)h_f(x,z).
\end{align}
Thus, ignoring the constant $\ell(\undery,\tfrac{\alpha}{2})$, the loss $\ell_\alpha$ can be viewed as the cost-sensitive error of the classifier $h_f$.

For $z\in\calZ$, we can rewrite $\gamma_{a,z}(\underf)$ as
\begin{align}
\notag
\gamma_{a,z}(\underf)
  &= \P[\underf(X)\ge z\given A=a] - \P[\underf(X)\ge z]
\\
\label{eq:gamma:hf}
  &= \E[h_f(X,z)\given A=a] - \E[h_f(X,z)]
\\
\notag
  &=\gamma_{a,z}(h_f),
\end{align}
completing the proof of the lemma.

\section{Iteration Complexity of Algorithm~\ref{alg:regred:sp}}

\begin{theorem}
\label{thm:alg:sp}
\Alg{regred:sp} terminates in at most $\frac{16 B^2\log(2\card{\calA}\cardZ+1)}{\nu^2}$ iterations. Furthermore,
if $Q$ is any feasible point of \eqref{eq:prob:sp:empirical} then the solution $\hQ$ returned by \Alg{regred:sp} satisfies:
\begin{align*}
   \hcost(\hQ)
   &\le
   \hcost(Q)+2\nu
\\
   \bigAbs{\hgamma_{a,z}(\hQ)}
   &\le\heps_{a}+\frac{2+2\nu}{B}
\quad
\text{for all $a\in\calA$, $z\in\calZ$.}
\end{align*}
\end{theorem}

\begin{proof}
This result is essentially a corollary of Theorem 1, and Lemmas 2 and 3 of \citet{agarwal2018reductions}. Specifically, we note that the constraints appearing in our problem~\eqref{eq:prob:sp:empirical} can be cast in their general framework along the lines of their Example 1, with a total of $2\cardA\cardZ$ constraints.
%
%
Following their Example 3, we obtain that the maximal constraint violation $\rho$, needed in their Theorem 1, is at most 2. We further observe that the violation of the i.i.d.\ structure by explicit averaging over $z$ values does not impact their optimization analysis in any way.
Therefore, their Theorem 1 with $\rho=2$ implies that our Algorithm~\ref{alg:regred:sp} finds a $\nu$-approximate saddle point of the Lagrangian in at most $\frac{16 B^2\log(2\card{\calA}\cardZ+1)}{\nu^2}$ iterations as desired.

To bound $\hcost(\hQ)$ and $\bigAbs{\hgamma_{a,z}(\hQ)}$ we appeal to their Lemmas 2 and 3. First note that
their approach applies to the objective of our problem~\eqref{eq:prob:sp:empirical} as long as the costs $c(y,z)$ are in $[0,1]$ (see their footnote 4). However, in our case, these can be in $[-1,1]$ (see Eq.~\ref{eqn:costs}). This does not affect their Theorem 1 and Lemma 2, but their Lemma 3 now holds with the right-hand side equal to $\frac{2+2\nu}{B}$ instead of $\frac{1+2\nu}{B}$. Their Lemma 2 immediately yields the bound $\hcost(\hQ)\le\hcost(Q)+2\nu$, whereas the modified Lemma 3 yields the bound
$\bigAbs{\hgamma_{a,z}(\hQ)}
 \le\heps_{a}+\frac{2+2\nu}{B}$ for all $a,z$, finishing the proof.
%
%
%
\end{proof}

\section{Proof of Theorem~\ref{thm:main:sp}}
\label{app:thm:main:sp}

Before proving the theorem, we recall a standard definition of the Rademacher complexity of a class of functions, which plays an important role in our deviation bounds. Let $\calG$ be a class of functions $g:\calU\to\R$ over some space $\calU$. Then the (worst-case) Rademacher complexity of $\calG$ is defined as:
\begin{equation}
R_n(\calG)
\coloneqq
\sup_{u_1,\ldots,u_n \in \calU} \E_{\sigma} \left[\sup_{g \in \calG} \left|\frac{1}{n} \sum_{i=1}^n \sigma_i g(u_i)\right|\right],
\label{eqn:rademacher}
\end{equation}
where the expectation is over the i.i.d.\ random variables $\sigma_1,\ldots,\sigma_n$ with $\P(\sigma_i = 1) = \P(\sigma_i = -1) = 1/2$.

The Rademacher complexity of a class $\calG$ can be used to obtain uniform bounds of any Lipschitz continuous
transformations of $g\in\calG$ as follows:
\begin{lemma}
\label{lemma:dev}
Let $D$ be a distribution over a pair of random
variables $(S,U)$ taking values in $\calS\times\calU$.
Let $\calG$ be a class of functions $g:\calU\to[0,1]$, and let $\phi:\calS\times[0,1]\to[-1,1]$ be a contraction in its second argument, i.e.,
for all $s\in\calS$ and all $t,t'\in[0,1]$, $\abs{\phi(s,t)-\phi(s,t')}\le\abs{t-t'}$.
Then with probability at least $1-\delta$, for all $g\in\calG$,
  \[
  \BigAbs{
     \Ehat\bigBracks{\phi(S,g(U))}-\E\bigBracks{\phi(S,g(U))}
  }
  \le
     4R_n(\calG) + \frac{2}{\sqrt{n}} + \sqrt{\frac{2\ln(2/\delta)}{n}}
\enspace,
  \]
where the expectation is with respect to $D$ and the empirical expectation is based on $n$ i.i.d.\ draws from $D$. If $\phi$ is also linear in its second argument then a tighter bound holds, with $4R_n(\calG)$ replaced by $2R_n(\calG)$.
\end{lemma}
\begin{proof}
Let $\Phi\coloneqq\set{\phi_g}_{g\in\calG}$ be the class of functions $\phi_g: (s,u)\mapsto \phi(s,g(u))$.
By Theorem 3.2 of \citet{BoucheronBoLu05}, we then
have with probability at least $1-\delta$, for all $g$,
\begin{equation}
\label{eqn:phi:bound}
  \BigAbs{
     \Ehat\bigBracks{\phi(S,g(U))}-\E\bigBracks{\phi(S,g(U))}
  }
  =
  \BigAbs{\Ehat[\phi_g]-\E[\phi_g]}
  \le
  2 R_n(\Phi)
  +
  \sqrt{\frac{2\ln(2/\delta)}{n}}
\enspace.
\end{equation}
We will next bound $R_n(\Phi)$ in terms of $R_n(\calG)$. For a fixed tuple $(s_1,u_1),\dotsc,(s_n,u_n)$, we have
\begin{align*}
 \E_\sigma\Bracks{
     \sup_{g \in \calG}  \Abs{\sum_{i=1}^{n} \sigma_i \phi(s_i,g(u_i)) }
 }
&
\le
 \E_\sigma\Bracks{
     \sup_{g \in \calG}  \Abs{\sum_{i=1}^{n} \sigma_i \BigParens{\phi(s_i,g(u_i))-\phi(s_i,0)} }
 }
 +
 \sqrt{n}
\\
&
\le
 2\E_\sigma\Bracks{
     \sup_{g \in \calG}  \Abs{\sum_{i=1}^{n} \sigma_i g(u_i)}
 }
 +
 \sqrt{n}
\end{align*}
where the first inequality follows from Theorem 12(5) of~\citet{BartlettMe02} and the last inequality follows from the contraction principle of~\citet{LedouxTa91}, specifically their Theorem 4.12.
Dividing by $n$ and taking a supremum over $(s_1,u_1),\dotsc,(s_n,u_n)$ yields the bound
\[
  R_n(\Phi)\le 2R_n(\calG)+\frac{1}{\sqrt{n}}.
\]
Together with the bound~\eqref{eqn:phi:bound}, this proves the lemma for an arbitrary contraction $\phi$. If $\phi$ is linear in its second argument, we get a tighter bound by invoking Theorem 4.4 of~\citet{LedouxTa91} instead of their Theorem 4.12.
\end{proof}

Our proof largely follows the proof of Theorems~2 and~3 of \citet{agarwal2018reductions}. We first use \Lemma{dev} to show that by solving the empirical problem~\eqref{eq:prob:sp:empirical}, we also obtain an approximate solution of the corresponding population problem:
\begin{align}
\label{eq:prob:sp}
  \min_{\mathclap{Q\in\Delta(\calH)}}\;\cost(Q)
\text{ s.t. }
  \bigAbs{\gamma_{a,z}(Q)}\le\eps_a
\;\;
  \text{$\forall a\in\calA, z\in\calZ$.}
%
\end{align}
The theorem will then follow by invoking the equivalence between problem~\eqref{eq:prob:sp} and the discretized
fair regression~\eqref{eq:fair-reg:sp:apx}, and adding up various approximation errors.

\paragraph{Bounding empirical deviations in the cost and constraints.}
To bound the deviations in the cost, we need to be a bit careful, because the definition of $\hcost$ mixes
the empirical expectation over the data with the averaging over $z\in\calZ$. For the analysis, we therefore
define
\[
   \hcost_z(h)=\Ehat\bigBracks{c(\underY,z)h(X,z)}
\quad
\text{and}
\quad
   \cost_z(h)=\E\bigBracks{c(\underY,z)h(X,z)}.
\]
Since $c(\underY,z)\in[-1,1]$, we can invoke \Lemma{dev} with $S=c(\underY,z)$, $U=(X,z)$, $\calG=\calH$, and $\phi(s,t)=st$ to obtain that with probability at least $1-\delta/2$ for all $z\in\calZ$ and all $h\in\calH$
\[
  \BigAbs{\hcost_z(h)-\cost_z(h)}
  \le
  2R_n(\calH)+\frac{2}{\sqrt{n}}+\sqrt{\frac{2\ln(4N/\delta)}{n}}
  =
  \otil(n^{-\beta})
,
\]
where the last equality follows by Assumption~\ref{ass:samples} and the setting $\cardZ\propto n^\beta$.
Taking an average over $z\in\calZ$ and a convex combination according to any $Q\in\Delta(\calH)$, we obtain by Jensen's inequality
that with probability at least $1-\delta/2$ for all $Q\in\Delta(\calH)$
\begin{equation}
\label{eq:dev:cost}
  \bigAbs{\hcost(Q)-\cost(Q)}
  =
  \otil(n^{-\beta})
.
\end{equation}

To bound the deviations in the constraints, we invoke \Lemma{dev} with $S=1$, $U=(X,z)$, $\calG=\calH$, and $\phi(s,t)=st$, but apply it to the data distribution conditioned on $A=a$. We thus obtain that with with probability at least $1-\delta/2$ for all $a\in\calA$, $z\in\calZ$, and $h\in\calH$
\[
  \BigAbs{\hgamma_{a,z}(h)-\gamma_{a,z}(h)}
  \le
  2R_{n_a}(\calH)+\frac{2}{\sqrt{n_a}}+\sqrt{\frac{2\ln(4\cardA N/\delta)}{n_a}}
.
\]
By Jensen's inequality this also means
that with probability at least $1-\delta/2$ for all $a\in\calA$, $z\in\calZ$, and $Q\in\Delta(\calH)$
\begin{equation}
\label{eq:dev:gamma}
  \bigAbs{\hgamma_{a,z}(Q)-\gamma_{a,z}(Q)}
  \le
  2R_{n_a}(\calH)+\frac{2}{\sqrt{n_a}}+\sqrt{\frac{2\ln(4\cardA N/\delta)}{n_a}}
.
\end{equation}
In the remainder of the analysis, we assume that \Eqs{dev:cost}{dev:gamma} both hold, which occurs with probability at least $1-\delta$ by the union bound.

\paragraph{Putting it all together.}

Given the settings of $\nu$, $B$ and $N$, we obtain by \Thm{alg:sp} that Algorithm~\ref{alg:regred:sp} terminates in $\order\bigParens{n^{4\beta}\ln (n^\beta\card{\calA})}$ iterations, as desired, and returns a distribution $\hQ$ which compares favorably
with any feasible point $Q$ of the empirical problem~\eqref{eq:prob:sp:empirical}, meaning that for
any such $Q$, we have
\begin{align}
\label{eq:sp:app:1}
   \hcost(\hQ)
   &\le
   \hcost(Q)+\order(n^{-\beta})
\\
\label{eq:sp:app:2}
   \bigAbs{\hgamma_{a,z}(\hQ)}
   &\le\heps_{a}+\order(n^{-\beta})
\quad
\text{for all $a\in\calA$, $z\in\calZ$.}
\end{align}
Now bounding $\hcost(\hQ)$ and $\hcost(Q)$ in \Eq{sp:app:1} via the uniform convergence bound~\eqref{eq:dev:cost}, we obtain
\begin{align}
\label{eq:sp:app:3}
   \cost(\hQ)
&
   \le
   \cost(Q)+\otil(n^{-\beta}).
\intertext{%
Similarly, bounding $\hgamma_{a,z}(\hQ)$ via the bound~\eqref{eq:dev:gamma} and
$\heps_a\le\eps_a+\otil(n_a^{-\beta})$ via Assumption~\ref{ass:samples} and our setting of $C'$, we obtain
}
\label{eq:sp:app:4}
   \bigAbs{\gamma_{a,z}(\hQ)}
&
   \le\eps_{a}+\otil(n_a^{-\beta})
\quad
\text{for all $a\in\calA$, $z\in\calZ$.}
\end{align}

Above, we assumed that $Q$ was a feasible point of the empirical problem~\eqref{eq:prob:sp:empirical}. However, assuming that \Eq{dev:gamma} holds, any
feasible solution of the population problem~\eqref{eq:prob:sp} is also feasible in
the empirical problem~\eqref{eq:prob:sp:empirical} thanks to our setting of $C'$. Thus,
\Eqs{sp:app:3}{sp:app:4} show that the solution $\hQ$ is approximately feasible and approximately optimal in the population problem~\eqref{eq:prob:sp}. It remains to relate $\hQ$ to the original fair regression problem~\eqref{eq:fair-reg:sp}.

First, by \Lemma{red:hf} and \Eqs{sp:app:3}{sp:app:4}, we can interpret $\hQ$ as a distribution over the set of discretized regressors $\underF$ and obtain that for all $\underQ\in\Delta(\underF)$ that
are feasible in the discretized fair regression problem~\eqref{eq:fair-reg:sp:apx}, we have
\begin{align}
\label{eq:sp:app:5}
   \loss_\alpha(\hQ)
&
   \le
   \loss_\alpha(\underQ)+\otil(n^{-\beta})
\\
\label{eq:sp:app:6}
   \bigAbs{\gamma_{a,z}(\hQ)}
&
   \le\eps_{a}+\otil(n_a^{-\beta})
\quad
\text{for all $a\in\calA$, $z\in[0,1]$,}
\end{align}
where in \Eq{sp:app:6} we have expanded the domain $z\in\calZ$ to $z\in[0,1]$ thanks to \Eq{gamma:apx}. Finally, by substituting the solution $\underQ^*$ of problem~\eqref{eq:fair-reg:sp:apx} as $\underQ$ in \Eq{sp:app:5} and
applying \Thm{sp:apx}, we obtain that for any $Q^*\in\Delta(\calF)$ that is feasible in the discretized fair regression problem~\eqref{eq:fair-reg:sp:apx}, we have
\begin{equation}
\label{eq:sp:app:7}
   \loss(\hQ)
   \le
   \loss(Q^*)+\alpha+\otil(n^{-\beta})=\loss(Q^*)+\otil(n^{-\beta}),
\end{equation}
where the last equality follows by our setting $\alpha=1/N=O(n^{-\beta})$. The theorem now follows from \Eqs{sp:app:7}{sp:app:6}.


\section{Algorithm for Fair Regression under Bounded Group Loss}
\label{sec:app:alg:ep}

In this section we provide a detailed pseudocode of our algorithm for fair regression under \epshort, described at a high level in Section~\ref{sec:max-loss-bound}.

\begin{algorithm}[t]
\caption{Fair regression with \ep}
\label{alg:regred:ep}
  \begin{algorithmic}[1]
    \Statex{Input:~~training examples $\braces{(X_i,Y_i,A_i)}_{i=1}^n$, slacks in fairness constraint $\hzeta_{a}\in[0,1]$, bound $B$, convergence threshold $\nu$}
\smallskip
    \Statex{Define best-response functions:}
    \Statex{~\hphantom{Input:}~$\BestF(\vlambda)\coloneqq\argmin_{f\in\calF} \LE(f,\vlambda)$}
    \Statex{~\hphantom{Input:}~$\BestLambda(Q)\coloneqq\argmax_{\vlambda\ge 0,\,\norm{\vlambda}_1\le B} \LE(Q,\vlambda)$}
\smallskip
    \State{Set learning rate $\eta=\nu/(2B)$}
    \State{Set $\vtheta_1=\vzero\in\R^{\card{\calA}}$}
\smallskip
    \For{$t=1, 2, \ldots$}
\smallskip
    \ForAll{$a$} \hfill \textit{// Compute $\vlambda_t$ from $\vtheta_t$ and find the best response $f_t$}
    \State{$\lambda_{t,a}\gets B\,\exp\braces{\theta_{t,a}}/\left(1+\sum_{a}\exp\braces{\theta_{t,a}}\right)$}
    \EndFor
    \State{$f_t\gets\BestF(\vlambda_t)$}
\medskip
    \State{$\hQ_t\gets\frac{1}{t}\sum_{t'=1}^t f_{t'},
      \quad
            \hvlambda_t\gets\frac{1}{t}\sum_{t'=1}^t \vlambda_{t'}$ \hfill \textit{// Calculate the current approximate saddle point}}
\medskip
    \State{$\overline{\nu}\gets \LE\bigParens{\hQ_t,\BestLambda(\hQ_t)}\;-\;\LE(\hQ_t,\hvlambda_t)$
           \hfill \textit{// Check the suboptimality of $(\hQ_t,\hvlambda_t)$}}
    \State{$\underline{\nu}\gets \LE(\hQ_t,\hvlambda_t)\;-\;\LE\bigParens{\BestF(\hvlambda_t),\hvlambda_t}$}
\medskip
    \If{$\max\set{\overline{\nu},\underline{\nu}}\le\nu$}
           \hfill \textit{// Terminate if converged}
    \smallskip
        \If{$\hgammaE(\hQ_t) \leq \hzeta_a + \tfrac{1+2\nu}{B}$ for all $a \in \calA$}
        \State{\textbf{return} $\hQ_t$}
        \Else
        \State{\textbf{return $\nullval$}}
        \EndIf
    \EndIf
\medskip
    \State{Set $\vtheta_{t+1}=\vtheta_t+\eta\hvgammaE(f_t)-\eta\hvzeta$\hfill \textit{// Apply the exponentiated-gradient update}}
    \EndFor
\end{algorithmic}
\end{algorithm}

\section{Proof of Theorem~\ref{thm:main:ep}}

The analysis of Algorithm~\ref{alg:regred:ep} proceeds similarly to the analysis of Algorithm~\ref{alg:regred:sp}.

\paragraph{Iteration complexity.} Similarly to our analysis of Algorithm~\ref{alg:regred:sp} in \Thm{alg:sp},
we can appeal to Theorem 1, and Lemmas 2 and 3 of \citet{agarwal2018reductions}. While our objective and constraints are for the distributions $Q$ over $[0,1]$-valued predictors $f\in\calF$, whereas their analysis is for distributions over $\set{0,1}$-valued classifiers, we can still directly use their Theorem 1, and Lemmas 2 and 3, because they only rely on the bilinear structure of the Lagrangian with respect to $Q$ and $\vlambda$ and the boundedness of the objective and constraints, which all hold in our setting. The maximal constraint violation, needed in their Theorem 1, is $\rho\le 1$. Therefore, their Theorem 1 implies that Algorithm~\ref{alg:regred:ep} terminates in at most $4B^2\ln(\cardA + 1)/\nu^2=O(n^{4\omega}\ln\cardA)$ iterations and
finds a $\nu$-approximate saddle point $\hQ$ of problem~\eqref{eq:minmax:ep}, albeit sometimes it ends up returning $\nullval$ instead of $\hQ$. To prove the theorem we consider two cases.

\paragraph{Case I: There is a feasible solution $Q^*$ to the original problem~\eqref{eq:fair-reg:ep}.}
Given the settings of $\nu$ and $B$, and using Lemmas 2 and 3 of \citet{agarwal2018reductions}, we obtain that the $\nu$-approximate saddle point $\hQ$ of the empirical problem~\eqref{eq:minmax:ep} satisfies
\begin{align}
\label{eq:ep:app:1}
  \hloss(\hQ)
  &\le
  \hloss(Q) + 2\nu
\\
\label{eq:ep:app:2}
  \hgammaE_a(\hQ)
  &
  \le
  \hzeta_a + \frac{1+2\nu}{B}
  \quad
  \text{for all $a\in\calA$,}
\end{align}
for any distribution $Q$ feasible in the empirical problem~\eqref{eq:fair-reg:ep:emp}. \Eq{ep:app:2} implies that in this case the algorithm returns $\hQ\ne\nullval$. It remains to argue that statements similar
to~\eqref{eq:ep:app:1} and~\eqref{eq:ep:app:2} hold for true expectations rather than just empirical expectations.

To turn the statements~\eqref{eq:ep:app:1} and~\eqref{eq:ep:app:2} into matching population versions, we use the concentration result of \Lemma{dev} similarly as in the proof of \Thm{main:sp}. Specifically, we invoke \Lemma{dev} with $S=Y$, $U=X$, $\calG=\calF$, and $\phi(s,t)=\ell(s,t)$, separately for the data distribution (with failure probability $\delta/2$) and the data distribution
conditioned on each $A=a$ (with failure probabilities $\delta/(2\cardA)\,$). We thus obtain that with probability at least $1-\delta$, for all $Q\in\Delta(\calF)$,
\begin{align*}
  \bigAbs{\hloss(Q)-\loss(Q)}
  &\le
  4R_{n}(\calF)+\frac{2}{\sqrt{n}}+\sqrt{\frac{2\ln\parens{4/\delta}}{n}}
\\
  \bigAbs{\hgammaE_a(Q)-\gammaE_a(Q)}
  &\le
  4R_{n_a}(\calF)+\frac{2}{\sqrt{n_a}}+\sqrt{\frac{2\ln\parens{4\cardA/\delta}}{n_a}}
  \quad
  \text{for all $a\in\calA$.}
\end{align*}
By Assumption~\ref{ass:samples:ep} and our setting of $C'$, this implies that with probability at least $1-\delta$, for all $Q\in\Delta(\calF)$,
\begin{align}
\label{eq:ep:app:3}
  \bigAbs{\hloss(Q)-\loss(Q)}
  &=\otil(n^{-\omega})
\\
\label{eq:ep:app:4}
  \bigAbs{\hgammaE_a(Q)-\gammaE_a(Q)}
  &\le C' n_a^{-\omega}
  \quad
  \text{for all $a\in\calA$.}
\end{align}
We continue the analysis assuming that the uniform convergence bounds~\eqref{eq:ep:app:3} and~\eqref{eq:ep:app:4} both hold. Instantiating the bound~\eqref{eq:ep:app:3} for $\hloss(Q)$ and $\hloss(\hQ)$ in \Eq{ep:app:1} yields,
for any $Q$ feasible in the empirical problem~\eqref{eq:fair-reg:ep:emp},
\begin{align}
\label{eq:ep:app:5}
  \loss(\hQ)
  &\le
  \loss(Q) + 2\nu + \otil(n^{-\omega})
  =
  \loss(Q) + \otil(n^{-\omega}),
\intertext{%
where the last equality follows by our setting of $\nu$. Similarly, instatianting the bound~\eqref{eq:ep:app:4} for $\hgammaE_a(\hQ)$ in \Eq{ep:app:2} yields}
\label{eq:ep:app:6}
  \gammaE_a(\hQ)
  &
  \le
  \hzeta_a + \frac{1+2\nu}{B} + \otil(n_a^{-\omega})
  \le
  \zeta_a + \otil(n_a^{-\omega})
  \quad
  \text{for all $a\in\calA$,}
\end{align}
where the last inequality follows by our setting of $B$, $\nu$, and $C'$ as well as the bound
$\hzeta_a\le\zeta_a+C'n_a^{-\omega}$ from Assumption~\ref{ass:samples:ep}.

Above, we assumed that $Q$ was a feasible point of the empirical problem~\eqref{eq:fair-reg:ep:emp}. However, assuming that \Eq{ep:app:4} holds, any
feasible solution of the population problem~\eqref{eq:fair-reg:ep} is also feasible in
the empirical problem~\eqref{eq:fair-reg:ep:emp} thanks to our setting of $C'$. Thus,
\Eqs{ep:app:5}{ep:app:6} hold for any $Q^*$ feasible in~\eqref{eq:fair-reg:ep}, proving the theorem in this case.

\paragraph{Case II: There is no feasible solution $Q^*$ to the original problem~\eqref{eq:fair-reg:ep}.} In this case, the $\nu$-approximate saddle point $\hQ$ that the algorithm finds may still satisfy
\begin{equation}
\label{eq:ep:app:7}
  \hgammaE_a(\hQ)\le\hzeta_a + \frac{1+2\nu}{B}
  \quad
  \text{for all $a\in\calA$,}
\end{equation}
in which case the algorithm returns $\hQ$ and the theorem holds vacuously since here is no feasible point $Q^*$. If the found approximate saddle point does not satisfy \Eq{ep:app:7}, then the algorithm returns $\nullval$ and the theorem also holds.

\section{Details for Efficient Implementation of Algorithm~\ref{alg:regred:sp}}
\label{sec:app:impl}

On the high level, \Alg{regred:sp} keeps track of auxiliary vectors $\vtheta_t^+,\vtheta_t^-\in\R^{\cardA\cardZ}$. These are used to obtain the vector $\vlambda_t\in\R{2\cardA\cardZ}$ played by the $\vlambda$-player. The $Q$-player responds to $\vlambda_t$ by playing $h_t\in\calH$ (\Step{best:response}), which is used to update $\vtheta_t^\pm$. The average of $\vlambda_t$'s and $h_t$'s played so far is the candidate solution for the saddle-point problem
\begin{equation}
\label{eq:minmax}
\adjustlimits\min_{\;\;Q\in\Delta\;\;}\max_{\vlambda\ge\vzero,\,\norm{\vlambda}_1\le B} L(Q,\vlambda).
\end{equation}
The algorithm checks the suboptimality of this solution and terminates once the convergence $\nu$ is reached (\StepsRange{conv:check}).


The updates of
$\vtheta_t^+$ and $\vtheta_t^-$ (\Step{theta:update}) run in time $O(\cardA\cardZ)$ assuming that $\vgamma(h_t)$ has already been calculated. Similarly, the transformation of $\vtheta_t^+$ and $\vtheta_t^-$ to $\vlambda_t$ (\StepsRange{lambda:update}) runs in time $O(\cardA\cardZ)$, because the denominator is the same across all coordinates and so it needs to be computed only once. We next show that the remaining operations, except for the two $\BestH$ calls, run in time $\order(n\log n+\card{\calA}\cardZ)$.



\textbf{Computation of $\hcost(h_f)$ and $\hvgamma(h_f)$.} This computation is implicit in the calculation of Lagrangian in Steps~\ref{step:L:bestL}--\ref{step:L:bestH}, and also in the updates of $\vtheta_t^+$ and $\vtheta_t^-$ in \Step{theta:update}. By \Eq{loss:hf},
\[
  \hcost(h_f)=\Ehat\BigBracks{\ell_\alpha(Y,\underf(X))-\ell\bigParens{Y,\,\tfrac{\alpha}{2}}},
\]
so it can be calculated in time $\order(n)$ in a single pass over examples. To calculate $\hvgamma(h_f)$, we keep the training examples partitioned
into $\card{\calA}$ disjoint sets according to their protected attribute $A$. Let $n_{a}$ be the number of examples with $A=a$.
We sort these examples according to $f(X)$ in time $\order(n_{a}\log n_{a})$. Now going through these examples from the largest $f(X)$ value to the smallest allows us to calculate the conditional expectations
\[
  \Ehat[h_f(X,z)\given A=a] = \Phat[f(X)\ge z\given A=a]
\]
across all $z$ in time $\order(n_{a}+\cardZ)$. Since we need to do this for all $a\in\calA$,
the overall runtime is $\order(n\log n+\card{\calA}\cardZ)$. Finally, to
calculate $\hvgamma$ we also need expectations $\Ehat[h_f(X,z)]$, which can be obtained by taking weighted sums
of $\Ehat[h_f(X,z)\given A=a]$ in time $\order(\card{\calA}\cardZ)$. Altogether, the running time
of computing $\hcost(h_f)$ and $\hvgamma(h_f)$ is therefore $\order(n\log n+\card{\calA}\cardZ)$.

\textbf{Computation of $L(h_f,\vlambda)$ and $L(\hQ_t,\vlambda)$.} Lagrangian is evaluated in Steps~\ref{step:L:bestL}--\ref{step:L:bestH} to determine whether the algorithm has converged. Note that $L(Q,\vlambda)$ depends on $Q$ only through $\hcost(Q)$ and $\hvgamma(Q)$.
If we have already computed these, $L(Q,\vlambda)$ can be evaluated in time $\order(\card{\calA}\cardZ)$. To calculate $L(h_f,\vlambda)$ for an arbitrary
$h_f$, we first need to calculate $\hcost(h_f)$ and $\hvgamma(h_f)$, so the overall running time is $\order(n\log n+\card{\calA}\cardZ)$. To calculate
$L(\hQ_t,\vlambda)$, note that $\hcost(\hQ_t)=\frac{1}{t}\sum_{t'=1}^t\hcost(h_{t'})$, so we can obtain $\hcost(\hQ_t)$ from $\hcost(\hQ_{t-1})$ at the cost of evaluation
of $\hcost(h_t)$ and similarly for $\hvgamma(\hQ_t)$. Therefore, the first evaluation of the form $L(\hQ_t,\vlambda)$ takes time $\order(n\log n+\card{\calA}\cardZ)$
and each consequent evaluation takes time $\order(\card{\calA}\cardZ)$.

\textbf{Computation of $\BestLambda(\hQ_t)$.} The best response of the $\vlambda$-player is used in \Step{L:bestL} to determine the suboptimality of the current solution. Given an arbitrary $Q$, $\BestLambda(Q)$ returns
$\vlambda$ maximizing $L(Q,\vlambda)$ over $\vlambda\ge\vzero,\norm{\vlambda}_1\le B$. By first-order optimality, the optimizing $\vlambda$ is either $\vzero$ or puts all of its mass on the most
violated constraint among $\hgamma_{a,z}(Q)\le\heps_{a}$, $\hgamma_{a,z}(Q)\ge-\heps_{a}$. In particular,
let $\ve_{a,z}^+$ and $\ve_{a,z}^-$ denote the basis vectors corresponding to coordinates $\lambda_{a,z}^+$ and $\lambda_{a,z}^-$. The call to $\BestLambda(Q)$ first calculates
\[
  (a^*\!,z^*)=\smash{\argmax_{(a,z)}\bigBracks{\abs{\hgamma_{a,z}(Q)}-\heps_{a}}}
\]
and then returns
\[
   \begin{cases}
   B\ve_{a^*,z^*}^+
      &\text{if $\hgamma_{a^*\!,z^*}(Q)>\heps_{a^*}$}
   \\
   B\ve_{a^*,z^*}^-
      &\text{if $\hgamma_{a^*\!,z^*}(Q)<-\heps_{a^*}$}
   \\
   \vzero
      &\text{otherwise.}
   \end{cases}
\]
Thus, $\BestLambda(\hQ_t)$ can be calculated in time $\order(\card{\calA}\cardZ)$ as long as we have $\hvgamma(\hQ_t)$, whose computation we have already accounted for within the computation of the Lagrangian of the form $L(\hQ_t,\vlambda)$.

\subsection{Details for Reduction to Cost-sensitive Classification}

By definition of $\hgamma_{a,z}(h)$, we have
\begin{align}
\notag
\sum_{a,z}\lambda_{a,z}\hgamma_{a,z}(h)
&=
  \sum_{a,z}\lambda_{a,z}\Parens{\frac{1}{p_{a}}\Ehat\bigBracks{h(X,z)\ind\set{A=a}}
  - \Ehat\bigBracks{h(X,z)}}
\\
\notag
&=
  \Ehat\BiggBracks{
  N\E_Z\BiggBracks{
     \sum_{a}\lambda_{a,Z}\,h(X,Z)\Parens{\frac{\ind\set{A=a}}{p_{a}} - 1}
  }}
\\
\label{eq:hgamma:1}
&=
  \Ehat\biggBracks{
  N\E_Z\biggBracks{
     \biggParens{
     \frac{\lambda_{A,Z}}{p_A} -\!\sum_a\lambda_{a,Z}
     }
     h(X,Z)
  }}.
\end{align}
In particular, \Eqs{hcost}{hgamma:1} imply that the minimization of $L(h,\vlambda)$ is indeed equivalent to minimizing the right-hand side of \Eq{c:lambda} as claimed.

\subsection{Details for Reduction to Least-squares Regression}
\label{sec:app:impl:ls}

The construction of the least-squares regression data set begins with \Eq{c:lambda}, which states
that for $h_f\in\calH$
\begin{align}
  \hcost(h_f)
  +
  \sum_{a,z}\lambda_{a,z}\hgamma_{a,z}(h_f)
\label{eq:c:lambda:2}
&\,=\,
  \frac1n\sum_{i\le n}\sum_{z\in\calZ}
     \frac1N c_{\vlambda}(\underY_i,A_i,z)h_f(X_i,z).
\end{align}
We will now rewrite the inner summation over $z$ as a loss function for the predictor $f$, parameterized by $\vlambda$. More precisely, for any $a\in\calA$ and $\ty\in\tcalY$, let the ``loss'' function for a prediction $u$ be defined as
\[
  g_{\vlambda}(\ty,a,u)=\sum_{z\in\calZ, z\le u}\frac1N c_{\vlambda}(\ty,a,z).
\]
Now consider a specific $i$ in \Eq{c:lambda:2}. Let $X\coloneqq X_i$, $A\coloneqq A_i$ and $\tY\coloneqq \underY_i$. Then the summation over $z$ for this specific $i$ can be written as
\begin{align*}
\sum_{z\in\calZ}
     \frac1N
     c_{\vlambda}(\tY,A,z)h_f(X,z)
&\,=\,
\sum_{z\in\calZ}
     \BigBracks{g_{\vlambda}(\tY,A,z)-g_{\vlambda}(\tY,A,z-\alpha)}\ind\set{f(X)\ge z}
\\
&\,=\,
  g_{\vlambda}(\tY,A,\floors{f(X)}_\alpha)
\\
&\,=\,
  g_{\vlambda}(\tY,A,f(X))
.
\end{align*}
%
Plugging this back into \Eq{c:lambda:2}, we see that the minimization of \Eq{c:lambda:2} over $h\in\calH$ is equivalent to the minimization of the empirical loss under $g_{\vlambda}$ among $f\in\calF$:
\[
  \min_{f\in\calF}
  \Bracks{
  \frac1n\sum_{i=1}^n
     g_{\vlambda}(\underY_i,A_i,f(X_i))
  }.
\]
Solving this problem directly seems to require access to a generic optimization oracle. We instead use a heuristic, where we first pick
\[
  U_i\in\argmin_{u\in[0,1]} g_{\vlambda}(\underY_i,A_i,u)
\]
and then seek to solve the least-squares regression problem
\[
  \min\sum_{i\le n} (U_i-f(X_i))^2
.
\]
To obtain the values $U_i$ we first calculate the values $c_{\vlambda}(\ty,a,z)$ across all $\ty$, $a$, and $z$
in the overall time $\order(\card{\calA}\card{\tcalY} N)$. Then, using the definition of $g_{\vlambda}$, the minimizer
of $g_{\vlambda}(\ty,a,u)$ over $u$ can be found in time $\order(N)$ for each specific value of $\ty$ and $a$, so
all the minimizers can be precalculated in time $\order(\card{\calA}\card{\tcalY} N)$. Thus, preparing the data set for
the least-squares reduction takes time $\order(\card{\calA}\card{\tcalY} N)$ and the resulting regression data set
is of size $n$.\looseness=-1

\subsection{Details for Reduction to Risk Minimization under $\ell$}
\label{app:reduction:ell}

The same reasoning that yielded the reduction to the least-squares regression can be used to derive
a reduction to risk minimization under any loss $\ell(y,u)$ that is convex in $u$. We again begin with computing the minimizers $U_i$ for each $g_{\vlambda}(\underY_i, A_i, f(X_i))$ term. Now suppose that there are two values $\tY_{i,1},\tY_{i,2} \in [0,1]$ such that $\frac{\partial}{\partial u} \ell(\tY_{i,1},U_i) \leq 0$ and $\frac{\partial}{\partial u} \ell(\tY_{i,2},U_i) \geq 0$ (if $\ell$ is non-smooth, we can pick arbitrary elements of the subdifferential set). If no such pair exists, then it is not possible to obtain $f(X_i) = U_i$ by minimizing $\ell$ over any distribution of examples since the gradient can never vanish at $U_i$. However, if such a pair exists, we can induce weights $W_{i,1}, W_{i,2}\in[0,1]$ such that $W_{i,1} + W_{i,2} = 1$ and
\[\textstyle
  W_{i,1}\,\frac{\partial}{\partial u} \ell(\tY_{i,1},U_i) + W_{i,2}\,\frac{\partial}{\partial u} \ell(\tY_{i,2},U_i) = 0.
\]
Hence, we create two weighted examples for each $(X_i,A_i,Y_i)$ triple in our dataset, and solve
\[
\min_{f \in \calF}\,\sum_{i=1}^n\BigBracks{W_{i,1}\,\ell(\tY_{i,1}, f(X_i)) + W_{i,2}\,\ell(\tY_{i,2},f(X_i))}.
\]
For instance, for logistic loss, we can always pick $\tY_{i,1} = 0$, $\tY_{i,2} = 1$ and $W_{i,2} = U_i$. The complexity of this reduction is identical to that of the least-squares reduction above, but the resulting risk minimization problem might be better aligned with the original problem as the experiments
in Section~\ref{sec:experiments} show.

\section{Additional Experimental Results}
\label{app:exp}

In this section, we include further details on our experimental
evaluation.

\paragraph{Evaluation on the training sets.} In Figure~\ref{train} we
include the training performances of our algorithm and the baseline
methods, including the SEO method and the unconstrained
regressors. Our method generally dominated or closely matched the baseline
methods. The SEO method provided solutions that were not Pareto
optimal on the law school data set.

\begin{figure*}
  \includegraphics[width=\textwidth, height=9cm]{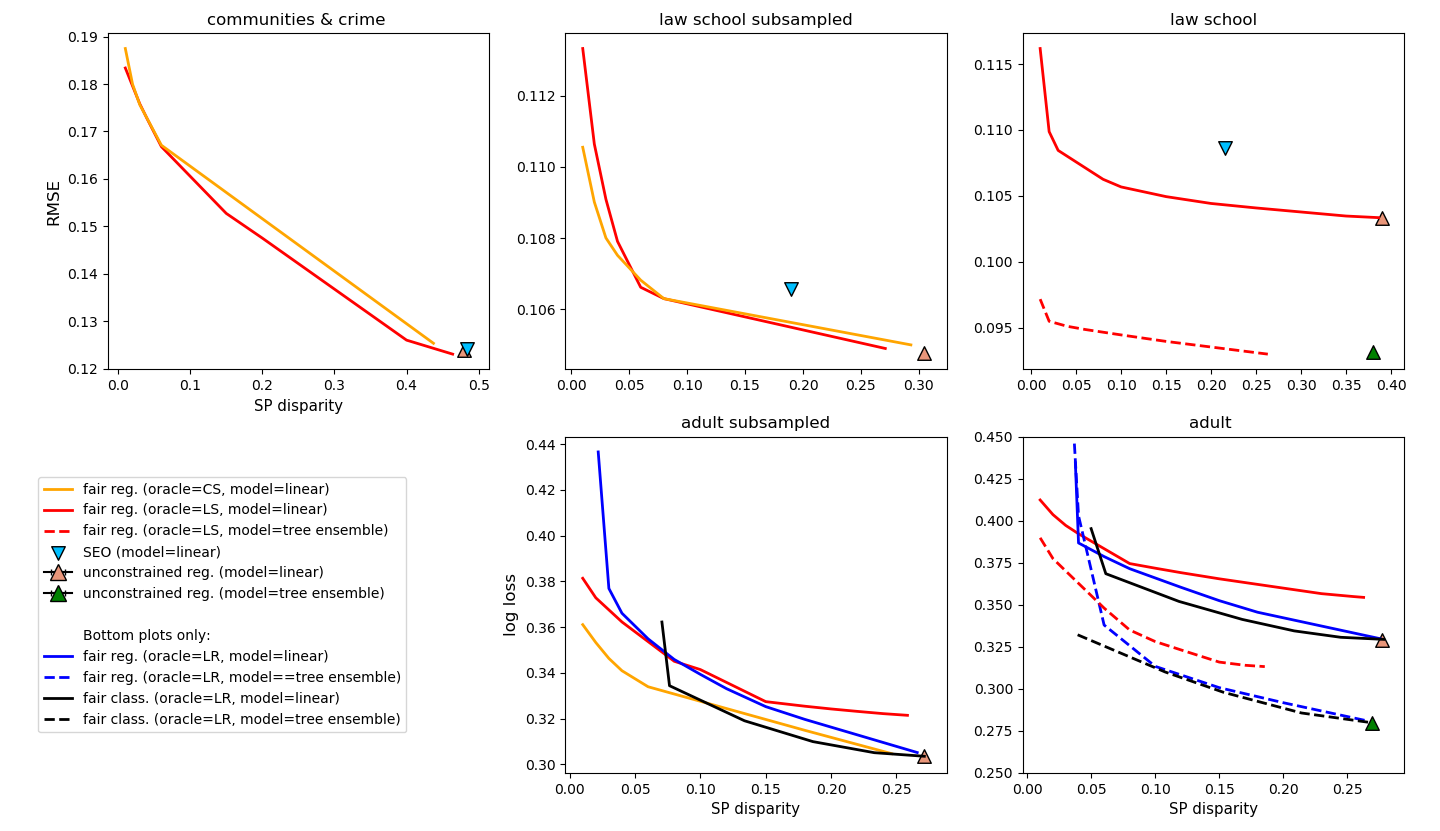}
\vspace{-8mm}
\caption{Training loss versus constraint violation with respect to
  DP. For our algorithm, we varied the fairness slackness parameter
  and plot the Pareto frontiers of the sets of returned predictors. For the logistic regression experiments, we also plot the Pareto frontiers of the sets of returned predictors given by fair classification reduction methods. }
\label{train}
\end{figure*}

\paragraph{Implementation of the cost-sensitive oracle.}
Given an instance of cost-sensitive classification problem, CS oracle
optimizes the equivalent weighted binary classification problem on the
data $\{W_i, X_i', Y_i\}_{i=1}^n$ with each $X_i' = (x_i, z_i)$ (see
Section \ref{sec:oracles} for the transformation). The oracle aims to
solve
\begin{equation}\label{hard}
\min_{h\in \calH}  \sum_{i=1}^n W_i \mathbf{1}\left\{ h(X_i') \neq Y_i \right\}.
\end{equation}
where every function $h$ in the class $\calH$ is parameterized by a
vector $\beta$ and defined as
$h(x, z) = \mathbf{1}\left\{\langle \beta, x\rangle \geq z\right\}$
for any input $(x, z)$. Instead of optimizing over the objective in
\eqref{hard}, we will consider the following minimization problem with
hinge loss. It will be convenient to consider the labels $Y_i$ take
values $\{\pm 1\}$ and each predictor $h$ predicts in $\{\pm
1\}$. Then the optimization becomes:
\[
  \min_{\beta} \sum_{i=1}^n W_i \max\left\{0, \tfrac{\alpha}{2} -
    \langle \beta, x_i\rangle Y_i\right\}
\]
Furthermore, we will optimize over $\beta$ with $\ell_\infty$ norm
bounded by 1. Then we can encode the optimization problem as a linear
program:
\begin{align*}
&\min_{\beta, t} \quad\sum_{i} t_i\\
\mbox{for all }i\in[n]:& \quad t_i \geq 0,\\
\mbox{for all }i\in[n]:& \quad t_i \geq \tfrac{\alpha}{2} - Y_i \, \langle \beta , x_i \rangle,\\
\mbox{for all }j\in[d]:& -1 \leq \beta_j \leq 1.
\end{align*}
In our experiments, this optimization problem was solved with the Gurobi Optimizer~\cite{gurobi}.

\paragraph{Runtime comparison.}
We performed a comparison on the running time of a single call of the
three supervised learning oracles. On a subsampled law school data set
with 1,000 examples, we ran the oracles to solve an instance of the
$\BestH$ problem, optimizing over either the linear models or tree
ensemble models. The details are listed in Table~\ref{table}. We also
compare the number of oracle calls for different specified values of
fairness slackness.

\begin{table}
\centering
\begin{tabular}{|l|l|l|}
\hline
Oracle & \begin{tabular}[c]{@{}l@{}}Model\\ class\end{tabular} & \begin{tabular}[c]{@{}l@{}}Runtime per call\\ (seconds)\end{tabular} \\ \hline
CS     & linear                                                & 18.94                                                                \\ \hline
LS     & linear                                                & 3.48                                                                 \\ \hline
LS     & tree ensemble                                         & 3.60                                                                 \\ \hline
LR     & linear                                                & 3.62                                                                 \\ \hline
LR     & tree ensemble                                         & 3.69                                                                 \\ \hline
\end{tabular}
\caption{Runtime comparison on the oracles over different model
  classes. We ran the oracles on a sub-sampled law school data set
  with 1,000 examples, using a machine with a 2.7 GHz Intel processor
  and 16GB memory.}
\label{table}
\end{table}

\begin{figure*}[!t]
  \includegraphics[width=\textwidth, height=4cm]{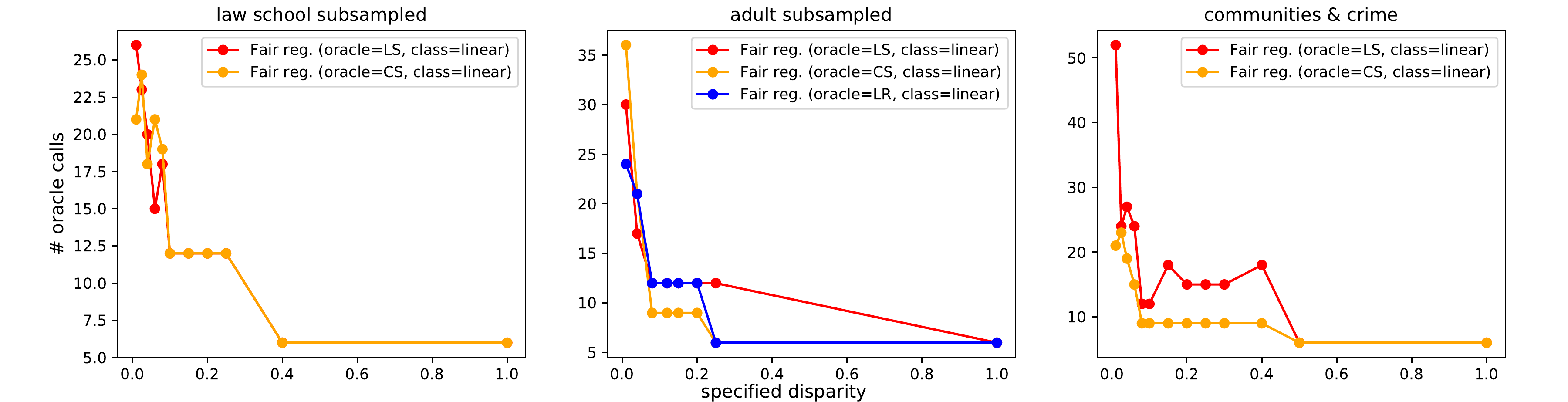}
\vspace{-8mm}
\caption{Number of oracle calls versus specified value of fairness slackness.}
\end{figure*}


\end{document}